\pdfoutput=1
\documentclass{article} 
\usepackage{times}
\usepackage{geometry}

\usepackage[utf8]{inputenc} 
\usepackage[T1]{fontenc}    
\usepackage[colorlinks]{hyperref}       
\usepackage{url}            
\usepackage{booktabs}       
\usepackage{amsfonts}       
\usepackage{nicefrac}       
\usepackage{microtype}      

\usepackage{mkolar_definitions}
\usepackage{url}
\usepackage{authblk}

\usepackage{graphicx}
\usepackage{subfigure}
\usepackage{multirow}
\usepackage{floatrow}
\newfloatcommand{capbtabbox}{table}[][\FBwidth]

\PassOptionsToPackage{numbers}{natbib}
\usepackage{natbib}

\usepackage{algorithm}
\usepackage{algorithmic}
\usepackage{tikz}
\usepackage{xspace}
\usepackage{graphicx}
\usepackage{eucal}

\usepackage{enumerate}
\usepackage{amsthm,amsmath,amssymb}
\usepackage{amsfonts}
\usepackage{nicefrac}
\usepackage{microtype}
\usepackage{mathtools}
\newcommand{\amo}{\textsc{AMO}\xspace}
\newcommand{\OP}{\textup{\textsc{OP}}\xspace}
\newcommand{\semimonster}{\textsc{VCEE}\xspace}
\newcommand{\eels}{\textsc{EELS}\xspace}
\newcommand{\linucb}{\textsc{LinUCB}\xspace}
\newcommand{\epsgreedy}{\textsc{$\epsilon$-Greedy}\xspace}
\newcommand{\sklearn}{\texttt{scikit-learn}\xspace}

\newcommand{\simplex}{\ensuremath{\Delta_{\le}}}

\newcommand{\pmin}{\ensuremath{p_{\textup{min}}}}
\newcommand{\ns}{\ensuremath{n_\star}}
\newcommand{\ls}{\ensuremath{\lambda_\star}}
\newcommand{\ws}{\ensuremath{w^\star}}
\newcommand{\ts}{\ensuremath{t^\star}}

\newcommand{\normratio}{\frac{\norm{\ws}_2^2}{\norm{\ws}_1}}
\newcommand{\invnormratio}{\frac{\|\ws\|_1}{\|\ws\|_2^2}}
\newcommand{\Reghat}{\widehat{\textup{Reg}}}
\newcommand{\Reg}{\textup{Reg}}

\newcommand{\RE}[2]{\textup{RE}{(#1\Vert#2)}}
\newcommand{\bigRE}[2]{\textup{RE}{\bigl(#1\bigm\Vert#2\bigr)}}

\newcommand{\logN}{\log N}
\newcommand{\Eq}[1]{Eq.~\eqref{eq:#1}}
\newcommand{\Thm}[1]{Theorem~\ref{thm:#1}}
\newcommand{\Lem}[1]{Lemma~\ref{lem:#1}}
\newcommand{\Unif}{\textup{Unif}}

\newcommand{\vone}{\mathbf{1}}

\newcommand{\set}[1]{\{#1\}}
\newcommand{\Set}[1]{\left\{#1\right\}}

\newcommand{\Braces}[1]{\left\{#1\right\}}

\newcommand{\bigBracks}[1]{\bigl[#1\bigr]}
\newcommand{\BigBracks}[1]{\Bigl[#1\Bigr]}

\newcommand{\BiggBracks}[1]{\Biggl[#1\Biggr]}
\newcommand{\Bracks}[1]{\left[#1\right]}
\newcommand{\Parens}[1]{\left(#1\right)}
\newcommand{\bigParens}[1]{\bigl(#1\bigr)}
\newcommand{\BigParens}[1]{\Bigl(#1\Bigr)}
\newcommand{\biggParens}[1]{\biggl(#1\biggr)}

\newcommand{\given}{\mathbin{\vert}}
\newcommand{\bigGiven}{\mathbin{\bigm\vert}}

\newcommand{\Norm}[1]{\left\lVert#1\right\rVert}
\newcommand{\bigNorm}[1]{\bigl\lVert#1\bigr\rVert}
\newcommand{\BigNorm}[1]{\Bigl\lVert#1\Bigr\rVert}
\newcommand{\norm}[1]{\lVert#1\rVert}
\newcommand{\abs}[1]{\lvert#1\rvert}
\newcommand{\Abs}[1]{\left\lvert#1\right\rvert}
\newcommand{\BigAbs}[1]{\Bigl\lvert#1\Bigr\rvert}
\newcommand{\biggAbs}[1]{\biggl\lvert#1\biggr\rvert}
\newcommand{\angles}[1]{\langle#1\rangle} 
\usepackage{comment}

\newcount\Comments  
\definecolor{darkgreen}{rgb}{0,0.5,0}
\definecolor{darkred}{rgb}{0.7,0,0}
\definecolor{teal}{rgb}{0.3,0.8,0.8}
\newcommand{\kibitz}[2]{\ifnum\Comments=1\textcolor{#1}{#2}\fi}
\newcommand{\miro}[1]{\kibitz{teal}{[MD: #1]}}

\newcounter{qcounter}
\newenvironment{enums}
 {\begin{list}{\arabic{qcounter}.}
 {\usecounter{qcounter} \setlength{\topsep}{0in} \setlength{\partopsep}{0in}
  \setlength{\parsep}{0in} \setlength{\itemsep}{\parskip}
  \setlength{\leftmargin}{0.21in} \setlength{\rightmargin}{0in}
  \setlength{\listparindent}{0.0in} \setlength{\labelwidth}{0.07in}
  \setlength{\labelsep}{0.1in} \setlength{\itemindent}{-0.04in}}}
 {\end{list}}

\newcommand{\order}{\ensuremath{O}}
\newcommand{\otil}{\ensuremath{\tilde{\order}}}
\newcommand{\citeapp}{\cite}

\newcommand{\figsqueeze}{\vspace{-6pt}}
\newcommand{\tablesqueeze}{\vspace{-10pt}}

\begin{document} 

\title{Contextual Semibandits via Supervised Learning Oracles}

\author[1]{
Akshay Krishnamurthy
\thanks{akshay@cs.umass.edu}}
\author[2]{
Alekh Agarwal
\thanks{alekha@microsoft.com}}
\author[2]{
Miroslav Dud\'{i}k
\thanks{mdudik@microsoft.com}}

\affil[1]{University of Massachusetts, Amherst, MA}
\affil[2]{Microsoft Research, New York, NY}

\maketitle

\begin{abstract}
We study an online decision making problem where on each round a learner chooses a list of items based on some side information, receives a scalar feedback value for each individual item, and a reward that is linearly related to this feedback.
These problems, known as contextual semibandits, arise in crowdsourcing, recommendation, and many other domains.
This paper reduces contextual semibandits to supervised learning, allowing us to leverage powerful supervised learning methods in this partial-feedback setting.
Our first reduction applies when the mapping from feedback to reward is known and leads to a computationally efficient algorithm with near-optimal regret.
We show that this algorithm outperforms state-of-the-art approaches on real-world learning-to-rank datasets, demonstrating the advantage of oracle-based algorithms.
Our second reduction applies to the previously unstudied setting when the linear mapping from feedback to reward is unknown. Our regret guarantees are
superior to prior techniques that ignore the feedback.
\end{abstract}

\section{Introduction}
\label{sec:introduction}

Decision making with partial feedback, motivated by applications
including personalized medicine~\citep{Robins89} and content
recommendation~\citep{Li10Contextual}, is receiving increasing
attention from the machine learning community.  These problems are
formally modeled as \emph{learning from bandit feedback},
where a learner repeatedly takes an action and observes a reward for
the action, with the goal of maximizing reward.  While bandit
learning captures many problems of interest, several applications
have additional structure: the action is combinatorial in nature and more
detailed feedback is provided.  For example, in internet applications,
we often recommend sets of items and record information about the
user's interaction with each individual item (e.g., click).  This additional feedback is unhelpful unless it relates to the
overall reward (e.g., number of clicks), and, as in previous work, we
assume a linear relationship.  This interaction is known as
the \emph{semibandit} feedback model.

Typical bandit and semibandit algorithms achieve reward that is
competitive with the single best fixed action, i.e., the best medical treatment
or the most popular news article for everyone.  This is often
inadequate for recommendation applications: while the most popular
articles may get some clicks, personalizing content to the users is
much more effective.  A better strategy is therefore to leverage
contextual information to learn a rich policy for selecting actions,
and we model this as \emph{contextual semibandits}.  In this setting,
the learner repeatedly observes a context (user features), chooses a
composite action (list of articles), which is an ordered tuple of
simple actions, and receives reward for the composite action
(number of clicks), but also feedback about each simple action (click).
The goal of the learner is to find a policy for mapping contexts to
composite actions that achieves high reward.

We typically consider policies in a large but constrained class, for
example, linear learners or tree ensembles. Such a class enables us to
learn an expressive policy, but introduces a computational challenge
of finding a good policy without direct enumeration. We build on the
supervised learning literature, which has developed fast algorithms for
such policy classes, including logistic regression and SVMs for linear classifiers and boosting for tree ensembles. We access the policy class exclusively through a
supervised learning algorithm, viewed as an oracle.

In this paper, we develop and evaluate oracle-based algorithms for the contextual semibandits problem.
We make the following contributions:

\begin{enums}
\item In the more common setting where the
  linear function relating the semibandit feedback to the reward
  is known, we develop a new algorithm, called \semimonster, that
  extends the oracle-based contextual bandit algorithm of
  \citet{agarwal2014taming}.  We show that
  \semimonster enjoys a regret bound between $\otil\bigl(\sqrt{KLT
    \logN}\bigr)$ and $\otil\bigl(L\sqrt{KT\logN}\bigr)$,
  depending on the combinatorial structure of the problem, when there
  are $T$ rounds of interaction, $K$ simple actions, $N$ policies, and composite actions have length $L$.\footnote{%
  Throughout the paper, the $\otil(\cdot)$ notation suppressed factors polylogarithmic in $K$, $L$, $T$ and $\logN$. 
  We analyze finite policy classes, but our work extends to infinite classes by
    standard discretization arguments.}
%
%
\semimonster can handle structured action spaces and makes $\otil(T^{3/2})$ calls to
  the supervised learning oracle.

\item We empirically evaluate this algorithm on two large-scale learning-to-rank datasets and compare with other contextual semibandit approaches.
These experiments comprehensively demonstrate that effective exploration over a rich policy class can lead to significantly better performance than existing approaches.
To our knowledge, this is the first thorough experimental evaluation of not only oracle-based semibandit methods, but of oracle-based contextual bandits as well.
%

\item When the linear function relating the feedback to the reward is unknown, we
develop a new algorithm called \eels. Our algorithm first learns the linear function by uniform exploration and then,
adaptively, switches to act according to an empirically optimal policy.
We prove an $\otil\left((LT)^{2/3}(K\logN)^{1/3}\right)$ regret bound by analyzing when to switch.
We are not aware of other computationally efficient procedures with a matching or better regret bound for this setting.
\end{enums}

\begin{table}[t]
\label{tab:comparison}
\begin{center}
\begin{tabular}{lccc}
\toprule
Algorithm & Regret & Oracle Calls & Weights $\ws$
\\
\midrule
\semimonster (Thm.~\ref{thm:semimonster}) & $\sqrt{KLT\logN}$ & $T^{3/2}\sqrt{K/(L\logN)}$ & known
\\
$\epsilon$-Greedy (Thm.~\ref{thm:eps_greedy}) & $(LT)^{2/3}(K\logN)^{1/3}$ & $1$ & known
\\
\citet{kale2010non} & $\sqrt{KLT\logN}$ & not oracle-based & known
\\
\addlinespace
\eels (Thm.~\ref{thm:eels}) & $(LT)^{2/3}(K\logN)^{1/3}$ & $1$ & unknown
\\
\citet{agarwal2014taming} & $L\sqrt{K^LT\logN}$ & $\sqrt{K^LT/\logN}$ & unknown
\\
\citet{swaminathan2016off} & $L^{4/3}T^{2/3}(K\logN)^{1/3}$ & $1$ & unknown
\\
\bottomrule
\end{tabular}%
\end{center}
\caption{Comparison of contextual semibandit algorithms for arbitrary policy classes, assuming
all rankings are valid composite actions.
  The reward is semibandit feedback weighted according to $\ws$.
  For known weights, we consider $\ws=\mathbf{1}$; for unknown
  weights, we assume
  $\norm{\ws}_2\le\order(\sqrt{L})$.}
\tablesqueeze
\end{table}

See Table~\ref{tab:comparison} for a comparison of our results with existing applicable bounds. 

\paragraph{Related work.}

There is a growing body of work on combinatorial bandit
optimization~\cite{cesa2012combinatorial,AudibertBuLu14} with
considerable attention on semibandit
feedback~\cite{GyorgyEtAl07,kale2010non,chen2013combinatorial,qin2014contextual,kveton2015tight}.
The majority of this research focuses on the non-contextual setting
with a known relationship between semibandit feedback and reward, and
a typical algorithm here achieves an $\tilde{O}(\sqrt{KLT})$ regret
against the best fixed composite action.  To our knowledge, only the
work of \citet{kale2010non} and \citet{qin2014contextual} considers
the contextual setting, again with known relationship.  The former
generalizes the Exp4 algorithm~\cite{EXP4} to
semibandits, and achieves $\tilde{O}(\sqrt{KLT})$
regret,\footnote{\citet{kale2010non} consider the favorable setting
where our bounds match, when uniform exploration is valid.}  but
requires explicit enumeration of the policies.  The latter
generalizes the LinUCB algorithm of~\citet{chu2011contextual} to
semibandits, assuming that the simple action feedback is
linearly related to the context.  This differs from our setting: we
make no assumptions about the simple action feedback.
  In our
experiments, we compare \semimonster against this LinUCB-style
algorithm and demonstrate substantial improvements.

We are not aware of attempts to learn a relationship between the
overall reward and the feedback on simple actions as we do
with \eels. While \eels uses least squares, as in LinUCB-style
approaches, it does so \emph{without} assumptions on the semibandit
feedback. Crucially, the covariates for its least squares problem are
observed \emph{after predicting a composite action} and not before,
unlike in LinUCB.

Supervised learning oracles have been used as a computational primitive
in many settings including
active learning~\cite{hsu2010algorithms}, contextual
bandits~\cite{rakhlin2016bistro,syrgkanis2016efficient,agarwal2014taming,dudik2011efficient},
and structured prediction~\cite{daume09searn}.
%



\section{Preliminaries}
\label{sec:prelims}

Let $\Xcal$ be a space of contexts and $\Acal$ a set of $K$ simple
actions. Let $\Pi \subseteq (\Xcal \rightarrow \Acal^L)$ be a finite set of
policies, $|\Pi| = N$, mapping contexts to composite actions. Composite actions, also called
rankings,
are tuples of $L$ distinct simple actions. In general, there are $K!/(K-L)!$
possible rankings, but they might not be valid in all contexts. The set
of valid rankings for a context $x$ is defined implicitly through the policy class as $\{\pi(x)\}_{\pi \in \Pi}$.

Let $\Delta(\Pi)$ be the set of
distributions over policies, and $\simplex(\Pi)$ be the set of
non-negative weight vectors over policies, summing to at most 1, which
we call subdistributions. Let $\mathbf{1}(\cdot)$ be the 0/1 indicator
equal to 1 if its argument is true and 0 otherwise.

In stochastic contextual semibandits, there is an unknown
distribution $\Dcal$ over triples $(x,y,\xi)$,
where $x$ is a context, $y\in[0,1]^K$ is the vector of \emph{reward features}, with entries indexed by simple actions as $y(a)$, and $\xi\in[-1,1]$ is the reward noise, $\EE[\xi|x,y] = 0$.
Given $y\in\RR^K$ and $A=(a_1,\dotsc,a_L)\in\Acal^L$, we write $y(A)\in\RR^L$ for the vector with entries $y(a_\ell)$.
The learner plays a $T$-round game. In each round, nature draws $(x_t,y_t,\xi_t) \sim \Dcal$
and reveals the context $x_t$. The learner selects a valid ranking $A_t =
(a_{t,1},a_{t,2}, \ldots, a_{t,L})$ and gets reward $r_t(A_t) =
\sum_{\ell=1}^L \ws_\ell y_t(a_{t,\ell}) + \xi_t$, where $\ws \in \RR^L$ is a possibly unknown but fixed weight vector.  The learner is
shown the reward $r_t(A_t)$ and the vector of reward features for the chosen simple actions $y_t(A_t)$, jointly referred to as
\emph{semibandit feedback}.

The goal is to achieve cumulative reward competitive with all $\pi \in
\Pi$. For a policy $\pi$, let $\Rcal(\pi)\coloneqq \EE_{(x,y,\xi) \sim
  \Dcal}\bigBracks{r\bigParens{\pi(x)}}$
  denote its expected reward, and let
$\pi^\star \coloneqq \argmax_{\pi \in \Pi}\Rcal(\pi)$ be the maximizer of expected
reward.  We measure performance of an algorithm via cumulative empirical
regret,
\begin{align}
\textrm{Regret} \coloneqq \sum_{t=1}^Tr_t(\pi^\star(x_t)) - r_t(A_t).
\end{align}
The performance of a policy $\pi$ is measured by its expected regret, $\Reg(\pi) \coloneqq \Rcal(\pi^\star) - \Rcal(\pi)$.

\begin{example}
\normalfont In personalized search, a learning system repeatedly responds to queries with rankings of search items.
This is a contextual semibandit problem where the query and user features form the context, the simple actions are search items, and the composite actions are their lists.
The semibandit feedback is whether the user clicked on each item, while the reward may be the \emph{click-based discounted cumulative gain} (DCG), which is a weighted sum of clicks, with position-dependent weights.
We want to map contexts to rankings to maximize DCG
and achieve a low regret.
\end{example}

We assume that our algorithms have access to a \emph{supervised
  learning oracle}, also called an \emph{argmax oracle},
  denoted \amo, that can find a policy with
the maximum empirical reward on any appropriate dataset.  Specifically, given a
dataset $D = \{x_i,y_i,v_i\}_{i=1}^n$ of contexts $x_i$, reward feature vectors $y_i\in\RR^K$ with rewards \emph{for all simple actions},
and weight vectors
$v_i\in\RR^L$, the oracle computes
\begin{align}
\amo(D)
\;\coloneqq\;
\argmax_{\pi \in \Pi}\sum_{i=1}^n \langle v_i,y_i(\pi(x_i))\rangle
\;=\;
\argmax_{\pi \in \Pi}\sum_{i=1}^n \sum_{\ell=1}^Lv_{i,\ell} y_i(\pi(x_i)_\ell),
\label{eqn:amo}
\end{align}
where $\pi(x)_\ell$ is the $\ell$th simple action that policy $\pi$
chooses on context $x$. The oracle is supervised as it assumes known
features $y_i$ for all simple actions whereas we only observe them for
chosen actions.
This oracle is the
structured generalization of the one considered in contextual
bandits~\cite{agarwal2014taming,dudik2011efficient} and can be implemented by any
structured prediction approach such as CRFs~\cite{lafferty01crf}
or SEARN~\cite{daume09searn}.


Our algorithms choose composite actions by sampling from a
distribution, which allows us to use \emph{importance weighting} to construct
unbiased estimates for the reward features $y$.  If on round~$t$, a
composite action $A_t$ is chosen with probability $Q_t(A_t)$, we
construct the importance weighted feature vector $\hat{y}_t$ with
components $\hat{y}_t(a) \coloneqq y_t(a)\mathbf{1}(a \in A_t)/Q_t(a
  \in A_t)$, which are unbiased estimators of $y_t(a)$.
For a policy $\pi$, we then define
empirical estimates of its reward and regret, resp., as
\begin{align*}
\eta_t(\pi,w) \coloneqq \frac{1}{t}\sum_{i=1}^t \langle w, \hat{y}_i(\pi(x_i))\rangle \quad \textrm{and} \quad \Reghat_t(\pi,w) \coloneqq \max_{\pi'} \eta_t(\pi',w) - \eta_t(\pi, w).
\end{align*}
%
By construction, $\eta_t(\pi,\ws)$ is an unbiased estimate of the expected reward $\Rcal(\pi)$, but
$\Reghat_t(\pi,\ws)$ is \emph{not} an unbiased estimate of the expected regret $\Reg(\pi)$.
We use $\hat{\EE}_{x \sim H}[\cdot]$ to denote empirical expectation
over contexts appearing in the history of interaction $H$.

Finally, we introduce \emph{projections} and \emph{smoothing} of distributions.
For any $\mu \in [0,1/K]$ and any subdistribution $P\in\simplex(\Pi)$, the smoothed and projected conditional subdistribution $P^\mu(A\given x)$ is
\begin{align}
  P^\mu(A\given x) \coloneqq (1-K\mu)\sum_{\pi \in \Pi}P(\pi)\mathbf{1}
  (\pi(x) = A) + K\mu U_x(A),
  \label{eq:smoothing}
\end{align}
where $U_x$ is a uniform distribution over a certain subset of valid rankings for context $x$,
designed to ensure that the probability of choosing
each valid simple action is large.
By mixing $U_x$ into our action selection, we limit the variance of reward feature estimates $\hat{y}$.
%
%
The lower bound on the simple action probabilities under $U_x$ appears in our analysis as $\pmin$, which is the largest number satisfying
\[
  U_x(a\in A)\ge\pmin/K
\]
for all $x$ and all simple actions $a$ valid for $x$.
Note that $\pmin=L$ when there are no restrictions on the action space
as we can take $U_x$ to be the uniform distribution over all rankings
and verify that $U_x(a \in A) = L/K$. In the worst case,
$\pmin=1$, since we can always find one valid ranking for each
valid simple action and let $U_x$ be the uniform distribution over
this set. Such a ranking can be found efficiently by a call to
\amo for each simple action~$a$, with the dataset
of a single point $(x, \mathbf{1}_a \in
\RR^K, \mathbf{1} \in \RR^L)$, where $\mathbf{1}_a(a') =
\mathbf{1}(a=a')$.

\section{Semibandits with known weights}
\label{sec:semimonster}

We begin with the setting where the weights $\ws$ are known, and
present an efficient oracle-based algorithm (\semimonster, see Algorithm~\ref{alg:semimonster}) that generalizes the algorithm
of \citet{agarwal2014taming}.

The algorithm, before each round $t$,
constructs a subdistribution $Q_{t-1} \in \simplex(\Pi)$, which is used to form the distribution $\tilde{Q}_{t-1}$ by placing the missing mass on the maximizer of empirical reward.
The composite action for the context $x_t$ is chosen according to the smoothed distribution $\tilde{Q}_{t-1}^{\mu_{t-1}}$ (see Eq.~\eqref{eq:smoothing}).
The subdistribution $Q_{t-1}$ is any solution to the
feasibility problem (OP),
which
balances exploration and exploitation
via the constraints in Eqs.~\eqref{eq:op_regret}
and~\eqref{eq:op_variance}.  \Eq{op_regret} ensures that
the distribution has low empirical regret.  Simultaneously, Eq.~\eqref{eq:op_variance}
ensures that the variance of the reward estimates $\hat{y}$ remains
sufficiently small for each policy $\pi$, which helps control the deviation between
empirical and expected regret, and implies that $Q_{t-1}$ has low
expected regret.  For each $\pi$, the variance constraint is
based on the empirical regret of $\pi$, guaranteeing
sufficient exploration amongst all good policies.

\begin{algorithm}[t]
\caption{\semimonster (Variance-Constrained Explore-Exploit) Algorithm}
\label{alg:semimonster}
\begin{algorithmic}[1]
\REQUIRE Allowed failure probability $\delta \in (0,1)$.
\STATE $Q_0 = 0$, the all-zeros vector. $H_0=\emptyset$. Define:
$\mu_t = \min\left\{1/2K, \sqrt{\ln(16t^2N/\delta)/(Kt\pmin)}\right\}$.
\FOR {round $t = 1, \ldots, T$}
\STATE Let $\pi_{t-1} = \argmax_{\pi \in \Pi} \eta_{t-1}(\pi, \ws)$ and $\tilde{Q}_{t-1} = Q_{t-1} + (1 - \sum_\pi Q_{t-1}(\pi))\mathbf{1}_{\pi_{t-1}}$.
\label{step:tildeQ}
\STATE Observe $x_t \in \Xcal$, play $A_t \sim
\tilde{Q}_{t-1}^{\mu_{t-1}}(\cdot\given x_t)$ (see Eq.~\eqref{eq:smoothing}), and observe
$y_t(A_t)$ and $r_t(A_t)$.
\STATE Define $q_t(a) = \tilde{Q}_{t-1}^{\mu_{t-1}}(a \in A \given x_t)$ for each $a$.
\STATE Obtain $Q_t$ by solving OP with $H_t = H_{t-1} \cup
\set{ (x_t, y_t(A_t), q_t(A_t) }$
and $\mu_t$.
\ENDFOR
\end{algorithmic}
\vspace{3pt}%
\hrule%
\begin{center}\textbf{Semi-bandit Optimization Problem (OP)}\end{center}
With history $H$ and $\mu\ge 0$, define $b_\pi \coloneqq \invnormratio \frac{\Reghat_t(\pi)}{\psi \mu \pmin}$ and $\psi \coloneqq 100$.
Find $Q \in \simplex(\Pi)$ such that:
\begin{align}
\textstyle
\label{eq:op_regret}
&
\sum_{\pi \in \Pi} Q(\pi)b_\pi \le 2KL/\pmin
\\[-3pt]
\label{eq:op_variance}
\forall \pi \in \Pi:\quad
&
\hat{\EE}_{x \sim H}\BiggBracks{\sum_{\ell=1}^L \frac{1}{Q^\mu(\pi(x)_\ell\in A\given x)}}
\le \frac{2KL}{\pmin} + b_\pi
\end{align}
\vspace{-6pt}%
\end{algorithm}

OP can be solved efficiently using \amo and a coordinate descent
procedure obtained by modifying the algorithm
of \citet{agarwal2014taming}. While the full algorithm and analysis
are deferred to Appendix~\ref{sec:optimization_full},
several key differences between \semimonster~and the algorithm
of \citet{agarwal2014taming} are worth highlighting.  One crucial
modification is that the variance constraint in
Eq.~\eqref{eq:op_variance} involves the marginal probabilities of the
simple actions rather than the composite actions as would be the most
obvious adaptation to our setting.  This change, based on using the
reward estimates $\hat{y}_t$ for simple actions, leads to
substantially lower variance of reward estimates for all policies and, consequently, an improved regret bound. Another important modification
is the new mixing distribution $U_x$ and the quantity $\pmin$.  For
structured composite action spaces, uniform exploration over the valid
composite actions may not provide sufficient coverage of each simple
action and may lead to dependence on the composite action space size,
which is exponentially worse than when $U_x$ is used.

The regret guarantee for Algorithm~\ref{alg:semimonster} is the following:

\begin{theorem}
For any $\delta \in (0,1)$, with probability at least $1-\delta$, \semimonster achieves regret $\otil\bigParens{\frac{\norm{\ws}_2^2}{\norm{\ws}_1}L\sqrt{KT\log(N/\delta)\,/\,\pmin}}$. Moreover, \semimonster can be efficiently implemented with $\otil\bigParens{T^{3/2}\sqrt{K\,/\,(\pmin \log(N/\delta))}}$ calls to a supervised learning oracle \amo.
\label{thm:semimonster}
\end{theorem}

In Table~\ref{tab:comparison}, we compare this result to other applicable regret bounds in the most common setting, where $\ws = \mathbf{1}$ and all rankings are valid ($\pmin=L$).
\semimonster enjoys a $\otil(\sqrt{KLT\logN})$ regret bound, which is the best bound amongst oracle-based approaches, representing an exponentially better $L$-dependence over the purely bandit feedback variant \cite{agarwal2014taming} and a polynomially better $T$-dependence over an $\epsilon$-greedy scheme (see Theorem~\ref{thm:eps_greedy} in Appendix~\ref{app:eps_greedy}).
This improvement over $\epsilon$-greedy is also verified by our experiments.
Additionally, our bound matches that of \citet{kale2010non}, who consider the harder adversarial setting but give an algorithm that requires
an exponentially worse running time, $\Omega(NT)$, and cannot be efficiently implemented with an oracle.

Other results address the non-contextual setting, where the optimal
bounds for both stochastic~\cite{kveton2015tight} and
adversarial~\cite{AudibertBuLu14} semibandits are
$\Theta(\sqrt{KLT})$.  Thus, our bound may be optimal when
$\pmin=\Omega(L)$. However, these results apply even without requiring
all rankings to be valid, so they improve on our bound by a $\sqrt{L}$
factor when $\pmin=1$.  This $\sqrt{L}$ discrepancy may not be
fundamental, but it seems unavoidable with some degree of uniform
exploration, as in all existing contextual bandit algorithms.
A promising avenue to
resolve this gap is to extend the work of Neu~\cite{neu2015explore},
which gives high-probability bounds in the noncontextual setting
without uniform exploration.


To summarize, our regret bound is similar to existing results on
combinatorial (semi)bandits but represents a significant improvement
over existing computationally efficient approaches.

\section{Semibandits with unknown weights}

We now consider a generalization of the contextual semibandit problem
with a new challenge: the weight vector $\ws$ is unknown.
This setting is substantially more difficult than the previous one, as
it is no longer clear how to use the semibandit feedback to optimize
for the overall reward.  Our result shows that the semibandit feedback
can still be used effectively, even when the transformation is
unknown.  Throughout, we assume that the true weight vector $\ws$ has
bounded norm, i.e., $\norm{\ws}_2\le B$.

One restriction required by our analysis is the ability to play any ranking.
Thus, all rankings must be valid in all contexts, which
is a natural restriction in domains such as information retrieval and recommendation.
The uniform distribution over all rankings is denoted $U$.

\begin{algorithm}[t]
\begin{algorithmic}[1]
\REQUIRE Allowed failure probability $\delta \in (0,1)$. Assume $\norm{\ws}_2\le B$.
\STATE Set $\ns \gets T^{2/3}(K\ln(N/\delta)/L)^{1/3} \max\{1, (B\sqrt{L})^{-2/3}\}$
\FOR{$t=1,\ldots, \ns$}
\STATE Observe $x_t$, play $A_t\sim U$ ($U$ is uniform over all rankings), observe $y_t(A_t)$ and $r_t(A_t)$.
\ENDFOR
\STATE Let $\hat{V} = \frac{1}{2\ns K^2}\sum_{t=1}^{\ns}
\sum_{a,b \in \Acal}\bigParens{y_t(a) - y_t(b)}^2\frac{\mathbf{1}(a,b \in
    A_t)}{U(a,b \in A_t)}$.
\STATE $\tilde{V} \gets 2\hat{V} + 3\ln(2/\delta)/(2\ns)$.
\STATE Set $\ls \gets \max\left\{6L^2\ln(4LT/\delta), (T\tilde{V}/B)^{2/3}\left(L\ln(2/\delta)\right)^{1/3}\right\}$.
\STATE Set $\Sigma \gets \sum_{t=1}^{\ns}y_t(A_t)y_t(A_t)^T$.
\label{step:Sigma}
\WHILE{$\lambda_{\min}(\Sigma) \le \ls$}
\STATE $t \gets t+1$. Observe $x_t$, play $A_t\sim U$, observe $y_t(A_t)$ and $r_t(A_t)$.
\STATE Set $\Sigma \gets \Sigma + y_t(A_t)y_t(A_t)^T$.
\label{step:Sigma:update}
\ENDWHILE
\STATE Estimate weights $\hat{w} \gets \Sigma^{-1}(\sum_{i=1}^ty_i(A_i)r_i(A_i))$ (Least Squares).
\STATE Optimize policy $\hat{\pi} \gets \argmax_{\pi \in \Pi} \eta_t(\pi, \hat{w})$ using importance weighted features.
\STATE For every remaining round: observe $x_t$, play $A_t = \hat{\pi}(x_t)$.
\end{algorithmic}
\caption{\eels (Explore-Exploit Least Squares)}
\end{algorithm}

We propose an algorithm that explores first and then, adaptively,
switches to exploitation.  In the exploration phase, we play rankings
uniformly at random, with the goal of accumulating enough
information to learn the weight vector $\ws$ for effective policy
optimization.  Exploration lasts for a variable length of
time governed by two parameters $\ns$ and $\ls$.  The $\ns$ parameter
controls the minimum number of rounds of the exploration phase and is
$O(T^{2/3})$, similar to $\epsilon$-greedy style
schemes~\cite{langford2008epoch}.  The adaptivity is implemented by
the $\ls$ parameter, which imposes a lower bound on the eigenvalues of
the 2nd-moment matrix of reward features observed during exploration.
As a result, we only transition to the exploitation phase after
this matrix has suitably large eigenvalues.  Since we make no
assumptions about the reward features, there is no bound on how many
rounds this may take. This is a departure from previous explore-first
schemes, and captures the difficulty of learning $\ws$ when we observe
the regression features only after taking an action.


After the exploration phase of $t$ rounds, we perform least-squares
regression using the observed reward features and the rewards to learn
an estimate $\hat{w}$ of $\ws$.  We use $\hat{w}$ and importance
weighted reward features from the exploration phase to find a policy
$\hat{\pi}$ with maximum empirical reward, $\eta_t(\cdot,\hat{w})$.  The
remaining rounds comprise the exploitation phase, where we play
according to $\hat{\pi}$.

The remaining question is how to set $\ls$, which governs the length of the exploration phase. The ideal setting uses the unknown parameter
$V\coloneqq\EE_{(x,y)\sim\Dcal} \Var_{a\sim\Unif(\Acal)}[y(a)]$
of the distribution $\Dcal$, where $\Unif(\Acal)$ is the uniform distribution over all simple actions.
We form an unbiased estimator $\hat{V}$ of $V$ and derive an upper bound $\tilde{V}$.
While the optimal $\ls$ depends on $V$, the upper bound $\tilde{V}$ suffices.

For this algorithm, we prove the following regret bound.
\begin{theorem}
For any $\delta \in (0,1)$ and $T \ge K\ln(N/\delta)/\min\{L, (BL)^{2}\}$, with probability at least $1-\delta$, \eels has regret $\otil\left(T^{2/3}(K\log(N/\delta))^{1/3}\max\{B^{1/3}L^{1/2}, BL^{1/6}\}\right)$.
\eels can be implemented efficiently with one call to the optimization oracle.
\label{thm:eels}
\end{theorem}
The theorem shows that we can achieve sublinear regret without dependence on the composite action space size even when the weights are unknown.
The only applicable alternatives from the literature are displayed in Table~\ref{tab:comparison}, specialized to $B = \Theta(\sqrt{L})$.
First, oracle-based contextual bandits \citep{agarwal2014taming} achieve a better $T$-dependence, but both the regret and the number of oracle calls grow exponentially with $L$.
Second, the deviation bound of~\citet{swaminathan2016off}, which exploits the reward structure but not the semibandit feedback, leads to an algorithm with regret that is polynomially worse in its dependence on $L$ and $B$ (see Appendix~\ref{app:comparisons}).
This observation is consistent with non-contextual results, which show that the value of semibandit information is only in $L$ factors~\cite{AudibertBuLu14}.

Of course \eels has a sub-optimal dependence on $T$, although this is the best we are aware of for a computationally efficient algorithm in this setting. It is an interesting open question to achieve $\textrm{poly}(K,L)\sqrt{T\logN}$ regret with unknown weights.

\section{Proof sketches}
\label{sec:proofs}
We next sketch the arguments for our theorems.
Full proofs are deferred to the appendices.

\textbf{Proof of Theorem~\ref{thm:semimonster}}: The result
generalizes Agarwal et. al~\cite{agarwal2014taming}, and the proof
structure is similar.  For the regret bound, we use
Eq.~\eqref{eq:op_variance} to control the deviation of the empirical
reward estimates which make up the empirical regret $\Reghat_t$. A
careful inductive argument leads to the following bounds:
\begin{align*}
\Reg(\pi) \le 2\Reghat_t(\pi) + c_0\normratio KL\mu_t \quad
\textrm{and} \quad \Reghat_t(\pi) \le 2\Reg(\pi) + c_0\normratio
KL\mu_t.
\end{align*}
Here $c_0$ is a universal constant and $\mu_t$ is defined in the pseudocode. 
\Eq{op_regret} guarantees low empirical regret when playing according to $\tilde{Q}_t^{\mu_t}$, and the above inequalities also ensure small population regret.
The cumulative regret is bounded by $\normratio KL\sum_{t=1}^T\mu_t$, which grows at the rate given in \Thm{semimonster}.
The number of oracle calls is bounded by the analysis of the number of iterations of coordinate descent used to solve OP, via a potential argument similar to~\citet{agarwal2014taming}.

\textbf{Proof of Theorem~\ref{thm:eels}}:
We analyze the exploration and exploitation phases individually, and then optimize $\ns$ and $\ls$ to balance these terms.
For the exploration phase, the expected per-round regret can be bounded by either $\norm{\ws}_2\sqrt{KV}$ or $\norm{\ws}_2\sqrt{L}$,
but the number of rounds depends on the minimum eigenvalue $\lambda_{\min}(\Sigma)$, with $\Sigma$ defined in Steps~\ref{step:Sigma} and~\ref{step:Sigma:update}.
However, the expected per-round 2nd-moment matrix, $\EE_{(x,y)\sim\Dcal,A\sim U}[y(A)y(A)^T]$, has all eigenvalues at least $V$. 
Thus, after $t$ rounds, we expect $\lambda_{\min}(\Sigma) \geq tV$, so exploration lasts about $\ls/V$ rounds, yielding roughly
\begin{align*}
\textrm{Exploration Regret} \le \frac{\ls}{V}\cdot\norm{\ws}_2 \min\{\sqrt{KV}, \sqrt{L}\}
.
\end{align*}
Now our choice of $\ls$ produces a benign dependence on $V$ and yields a $T^{2/3}$ bound. 


For the exploitation phase, we bound the error between the empirical reward estimates $\eta_t(\pi,\hat{w})$ and the true reward $\Rcal(\pi)$.
Since we know $\lambda_{\min}(\Sigma) \ge \ls$ in this phase, we obtain
\begin{align*}
\textrm{Exploitation Regret} \le T\norm{\ws}_2\sqrt{\frac{K\logN}{\ns}} + T\sqrt{\frac{L}{\ls}}\min\{\sqrt{KV}, \sqrt{L}\}
.
\end{align*}
The first term captures the error from using the importance-weighted
$\hat{y}$ vector, while the second uses a bound on the error
$\norm{\hat{w} - \ws}_2$ from the analysis of linear
regression (assuming $\lambda_{\min}(\Sigma) \geq \lambda_\star$).

This high-level argument ignores several important details.
First, we must show that using $\tilde{V}$ instead of the optimal choice $V$ in the setting of $\ls$ does not affect the regret.
Secondly, since the termination condition for the exploration phase depends on the random variable $\Sigma$, we must derive a high-probability bound on the number of exploration rounds to control the regret.
Obtaining this bound requires a careful application of the matrix Bernstein inequality to certify that $\Sigma$ has large eigenvalues.


\section{Experimental Results}
\label{sec:experiments}

Our experiments compare \semimonster with existing alternatives.  As
\semimonster generalizes the algorithm of \citet{agarwal2014taming},
our experiments also provide insights into oracle-based contextual
bandit approaches and this is the first detailed empirical study of
such algorithms. The weight vector $\ws$ in our datasets was known, so
we do not evaluate \eels.  This section contains a high-level
description of our experimental setup, with details on our
implementation, baseline algorithms, and policy classes deferred to
Appendix~\ref{sec:app_exp}.  Software is available at
\url{http://github.com/akshaykr/oracle_cb}.


\textbf{Data:} We used two large-scale learning-to-rank datasets: MSLR~\cite{mslr} and all folds of the Yahoo!\ Learning-to-Rank dataset~\cite{chapelle2011yahoo}.
Both datasets have over 30k unique queries each with a varying number of documents that are annotated with a relevance in $\{0,\ldots,4\}$.
Each query-document pair has a feature vector ($d=136$ for MSLR and $d=415$ for Yahoo!) that we use to define our policy class.
For MSLR, we choose $K=10$ documents per query and set $L=3$, while for Yahoo!, we set $K=6$ and $L=2$.
The goal is to maximize the sum of relevances of shown documents ($\ws = \mathbf{1}$) and the individual relevances are the semibandit feedback.
All algorithms make a single pass over the queries.


\textbf{Algorithms:} We compare \semimonster, implemented with an
epoch schedule for solving OP after $2^{i/2}$ rounds
(justified by~\citet{agarwal2014taming}), with two baselines. First is the
\epsgreedy approach~\cite{langford2008epoch}, with a
constant but tuned $\epsilon$.  This algorithm explores uniformly with
probability $\epsilon$ and follows the empirically best policy otherwise. The
empirically best policy is updated with the same $2^{i/2}$ schedule.

We also compare against a semibandit version of
\linucb~\cite{qin2014contextual}.  This algorithm models the
semibandit feedback as linearly related to the query-document features
and learns this relationship, while selecting composite actions using
an upper-confidence bound strategy.
Specifically, the algorithm
maintains a weight vector $\theta_t \in \RR^d$ formed by solving a
ridge regression problem with the semibandit feedback
$y_t(a_{t,\ell})$ as regression targets.
At round $t$, the algorithm uses document features $\{x_a\}_{a
  \in \Acal}$ and chooses the $L$ documents with highest
$x_a^T\theta_t + \alpha x_a^T\Sigma_t^{-1}x_a$ value.  Here,
$\Sigma_t$ is the feature 2nd-moment matrix and $\alpha$ is a tuning
parameter.  For computational reasons, we only update $\Sigma_t$ and
$\theta_t$ every 100 rounds.

\textbf{Oracle implementation:}
\linucb only works with a linear policy class. \semimonster and \epsgreedy work with arbitrary classes. Here, we consider three:
linear functions and depth-2 and depth-5 gradient boosted regression trees
(abbreviated Lin, GB2 and GB5).
Both GB classes use 50 trees. Precise details of how we instantiate the supervised learning oracle can be found in Appendix~\ref{sec:app_exp}.


\begin{figure}
\begin{center}
\includegraphics[width=1.0\textwidth]{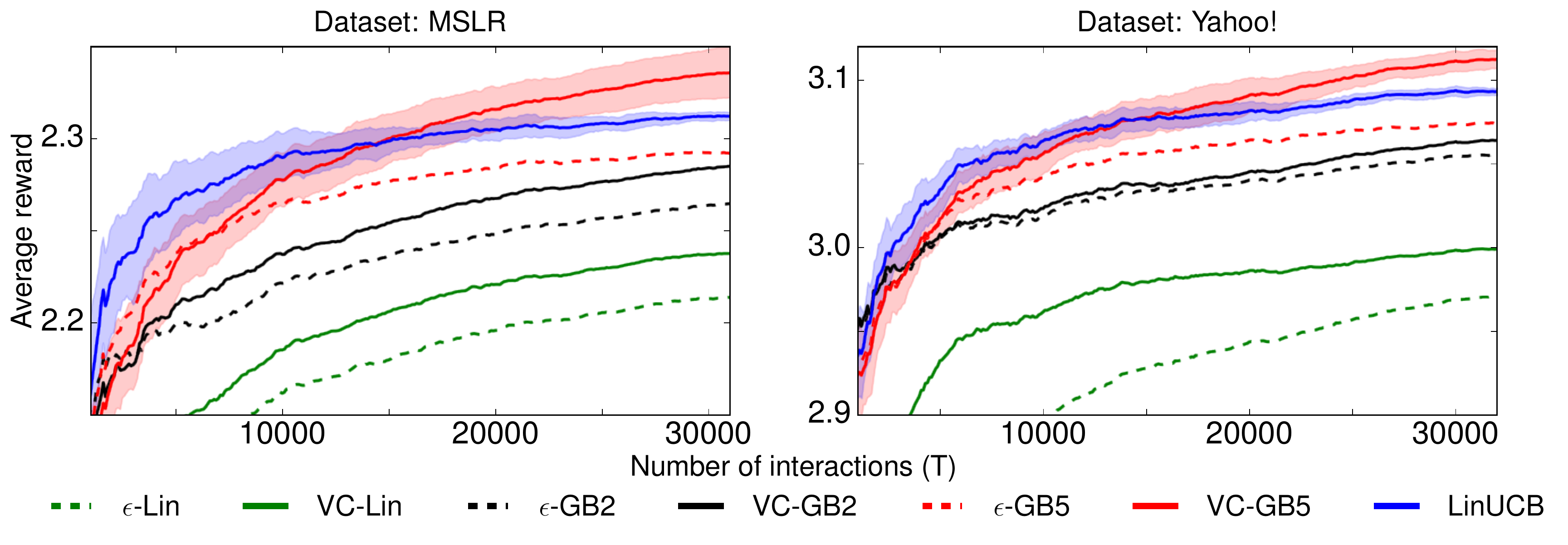}%
\end{center}
\figsqueeze%
\caption{Average reward as a function of number of interactions $T$ for \semimonster, \epsgreedy, and \linucb on MSLR (left) and Yahoo (right) learning-to-rank datasets.
}
\label{fig:experiment}
\figsqueeze%
\end{figure}

\textbf{Parameter tuning:} Each algorithm has a parameter governing
the explore-exploit tradeoff. For \semimonster, we set $\mu_t
= c\, \sqrt{1/KLT}$ and tune $c$, in \epsgreedy we tune $\epsilon$,
and in \linucb we tune $\alpha$.  We ran each algorithm for 10 repetitions, for
each of ten logarithmically spaced parameter values.

\textbf{Results:}
In Figure~\ref{fig:experiment}, we plot the average reward (cumulative reward up to round $t$ divided by $t$) on both datasets.
For each $t$, we use the parameter that achieves the best average reward across the 10 repetitions at that $t$.
Thus for each $t$, we are showing the performance of each algorithm tuned to maximize reward over $t$ rounds.
We found \semimonster was fairly stable to parameter tuning, so for VC-GB5 we just use one parameter value ($c=0.008$) for all $t$ on both datasets.
We show confidence bands at twice the standard error for just \linucb and VC-GB5 to simplify the plot.

Qualitatively, both datasets reveal similar phenomena.
First, when using the same policy class, \semimonster consistently outperforms \epsgreedy.
This agrees with our theory, as \semimonster achieves $\sqrt{T}$-type regret, while a tuned \epsgreedy achieves at best a $T^{2/3}$ rate.

Secondly, if we use a rich policy class, \semimonster can significantly improve on \linucb, the empirical state-of-the-art, and one of few practical alternatives to $\epsgreedy$.
Of course, since \epsgreedy does not outperform \linucb, the tailored exploration of \semimonster is critical.
Thus, the combination of these two properties is key to improved performance on these datasets.
\semimonster is the only contextual semibandit algorithm we are aware of that performs adaptive exploration \emph{and} is agnostic to the policy representation.
Note that \linucb is quite effective and outperforms \semimonster with a linear class. 
One possible explanation for this behavior is that \linucb, by directly modeling the reward, searches the policy space more effectively than \semimonster, which uses an approximate oracle implementation. 


\section{Discussion}
\label{sec:discussion}

This paper develops oracle-based algorithms for contextual semibandits both with known and unknown weights.
In both cases, our algorithms achieve the best known regret bounds for computationally efficient procedures.
Our empirical evaluation of \semimonster, clearly demonstrates the advantage of sophisticated oracle-based approaches over both parametric approaches and naive exploration.
To our knowledge this is the first detailed empirical evaluation of oracle-based contextual bandit or semibandit learning.
We close with some promising directions for future work: 
\begin{enums}
\item With known weights, can we obtain $\otil(\sqrt{KLT\logN})$ regret even with structured action spaces?
  This may require a new contextual bandit algorithm that does not use uniform smoothing.
\item With unknown weights, can we achieve a $\sqrt{T}$ dependence while exploiting semibandit feedback?
\end{enums}

\section*{Acknowledgements}
This work was carried out while AK was at Microsoft Research.
\newpage

\appendix
\section{Analysis of $\epsilon$-Greedy with Known Weights}
\label{app:eps_greedy}
We analyze the $\epsilon$-greedy algorithm (Algorithm~\ref{alg:eps_greedy}) in the known-weights setting when all rankings are valid, i.e., $\pmin = L$.
This algorithm is different from the one we use in our experiments in that it is an explore-first variant, exploring for the first several rounds and then exploiting for the remainder.
In our experiments, we use a variant where at each round we explore with probability $\epsilon$ and exploit with probability $(1-\epsilon)$.
This latter version also has the same regret bound, via an argument similar to that of~\citet{langford2008epoch}.

\begin{algorithm}[t]
\caption{$\epsilon$-Greedy for Contextual Semibandits with Known Weights}
\label{alg:eps_greedy}
\begin{algorithmic}[t]
\REQUIRE Allowed failure probability $\delta \in (0,1)$.
\STATE Set $n = T^{2/3}(K\ln(2N/\delta)/L)^{1/3}$.
\STATE Let $U$ be the uniform distribution over all rankings.
\STATE For $t=1,\ldots, n$, observe $x_t$, play $A_t \sim U$, observe $y_t(A_t)$ and $r_t(A_t)$.
\STATE Optimize policy $\hat{\pi} \gets \argmax_{\pi \in \Pi} \eta_n(\pi,\ws)$ using importance-weighted features.
\STATE For every remaining round: observe $x_t$, play $A_t = \hat{\pi}(x_t)$.
\end{algorithmic}
\end{algorithm}

\begin{theorem}
\label{thm:eps_greedy}
For any $\delta \in (0,1)$, when $T \ge K\ln(N/\delta)/L$, with probability at least $1-\delta$, the regret of Algorithm~\ref{alg:eps_greedy} is at most $\otil(\|\ws\|_2T^{2/3}(K\log(N/\delta))^{1/3}L^{1/6})$.
\end{theorem}
\begin{proof}
The proof relies on the uniform deviation bound similar to Lemma~\ref{lem:eels_reward_deviation}, which we use for the analysis of \eels.
We first prove that for any $\delta \in (0,1)$, with probability at least $1-\delta$, for all policies $\pi$, we have
\begin{align}
|\eta_n(\pi,\ws) - \Rcal(\pi)| \le \norm{\ws}_2\left(\sqrt{\frac{2K\ln(2N/\delta)}{n}} + \frac{2K}{3\sqrt{L}}\frac{\ln(2N/\delta)}{n}\right). \label{eq:eps_greedy_deviation}
\end{align}
This deviation bound is a consequence of Bernstein's inequality. The quantity on the left-hand side is the average of $n$ terms
\[
  \hat{y}_i(\pi(x_i))^T \ws - \EE_{x,y}[y(\pi(x))]^T \ws,
\]
all with expectation zero, because $\hat{y}$ is unbiased. The range of each term is bounded by the Cauchy-Schwarz inequality as
\begin{align*}
\|\ws\|_2 \|\hat{y}_i(\pi(x_i)) - \EE_{x,y}[y(\pi(x))]\|_2
\le \|\ws\|_2K/\sqrt{L},
\end{align*}
because under uniform exploration the coordinates of $\hat{y}_i(\pi(x_i))$ are bounded in $[0,K/L]$ while the coordinates of $y(\pi(x))$ are in $[0,1]$ and these are $L$-dimensional vectors.
The variance is bounded by the second moment, which we bound as follows:
\begin{align*}
\EE_{x,y,A}\Bracks{(\hat{y}(\pi(x))^T\ws)^2}
&\le \|\ws\|_2^2 \,\EE_{x,y,A}\!\Bracks{\sum_{l=1}^L\hat{y}(\pi(x_l))^2}\\
&\le \|\ws\|_2^2 \,\EE_{x,y,A}\!\Bracks{\sum_{l=1}^L \frac{K^2}{L^2}\mathbf{1}(\pi(x)_l \in A)}
 = \|\ws\|^2K,
\end{align*}
since $\EE_{x,y,A}[\mathbf{1}(\pi(x)_l \in A)] = L/K$ under uniform exploration.
Plugging these bounds into Bernstein's inequality gives the deviation bound of \Eq{eps_greedy_deviation}.

Now we can prove the theorem.
\Eq{eps_greedy_deviation} ensures that after collecting $n$ samples, the expected reward of the empirical reward maximizer $\hat{\pi}$ is close to $\max_\pi \Rcal(\pi)$, the best achievable reward.
The difference between these two is at most twice the right-hand side of the deviation bound.
If we perform uniform exploration for $n$ rounds, we are ensured that with probability at least $1-\delta$ the regret is at most
\begin{align*}
\textrm{Regret}
&\le  n \|\ws\|_2\sqrt{L} + 2(T-n)\|\ws\|_2\left(\sqrt{\frac{2K\ln(2N/\delta)}{n}} + \frac{2K}{3\sqrt{L}}\frac{\ln(2N/\delta)}{n}\right)\\
& \le n \|\ws\|_2\sqrt{L} + 3T\|\ws\|_2\left(\sqrt{\frac{K\ln(2N/\delta)}{n}} + \frac{K}{\sqrt{L}}\frac{\ln(2N/\delta)}{n}\right).
\end{align*}
For our setting of $n = T^{2/3}(K\ln(2N/\delta)/L)^{1/3}$, the bound is
\begin{align*}
4 \|\ws\|_2 T^{2/3}(K\ln(2N/\delta))^{1/3}L^{1/6} + 3\|\ws\|_2T^{1/3}(K\ln(2N/\delta))^{2/3}L^{-1/6}.
\end{align*}
Under the assumption on $T$, the second term is lower order, which proves the result.
\end{proof}
{
\section{Comparisons for \eels}
\label{app:comparisons}

In this section we do a detailed comparison of our Theorem~\ref{thm:eels} to the paper of~\citet{swaminathan2016off}, which is the most directly applicable result.
We use notation consistent with our paper.

\citet{swaminathan2016off} focus on off-policy evaluation in a more challenging setting where no semibandit feedback is provided.
Specifically, in their setting, in each round, the learner observes a context $x \in \Xcal$, chooses a composite action $A$ (as we do here) and receives reward $r(A) \in [-1,1]$.
They assume that the reward decomposes linearly across the action-position pairs as
\begin{align*}
\EE[r(A) | x,A] = \sum_{\ell=1}^L \phi_x(a_\ell,\ell).
\end{align*}
With this assumption, and when exploration is done uniformly, they provide off-policy reward estimation bounds of the form
\begin{align*}
|\eta_n(\pi) - \Rcal(\pi)| \le \order\Parens{\sqrt{\frac{KL\ln(1/\delta)}{n}}}.
\end{align*}
This bound holds for any policy $\pi: \Xcal \rightarrow \Acal^L$ with probability at least $1-\delta$ for any $\delta \in (0,1)$.
(See Theorem 3 and the following discussion in~\citet{swaminathan2016off}.) 
Note that this assumption generalizes our unknown weights setting, since we can always define $\phi_x(a,j) = \ws_jy(a)$.

To do an appropriate comparison, we first need to adjust the scaling of the rewards.
While~\citet{swaminathan2016off} assume that rewards are bounded in
$[-1,1]$, we only assume bounded $y$'s and bounded noise. Consequently, we
need to adjust their bound to incorporate this scaling.  If the
rewards are scaled to lie in $[-R,R]$, their bound becomes
\begin{align*}
|\eta_n(\pi) - \Rcal(\pi)| \le \order\Parens{R\sqrt{\frac{KL\ln(1/\delta)}{n}}}.
\end{align*}

This deviation bound can be turned into a low-regret algorithm by exploring for the first $n$ rounds, finding an empirically best policy, and using that policy for the remaining $T-n$ rounds.
Optimizing the bound in $n$ leads to a $T^{2/3}$-style regret bound:
\begin{fact}
\label{fact:slates}
The approach of~\citet{swaminathan2016off} with rewards in $[-R,R]$ leads to an algorithm with regret bound
\begin{align*}
\order\left(RT^{2/3}(KL\logN)^{1/3}\right).
\end{align*}
\end{fact}

This algorithm can be applied as is to our setting, so it is worth comparing it to \eels.
According to Theorem~\ref{thm:eels}, \eels has a regret bound
\begin{align*}
\order\left(T^{2/3}(K\logN)^{1/3}\max\{B^{1/3}L^{1/2}, BL^{1/6}\}\right).
\end{align*}
The dependence on $T, K$, and $\logN$ match between the two algorithms, so
we are left with $L$ and the scale factors $B,R$.  This comparison is
somewhat subtle and we use two different arguments.  The first finds a
conservative value for $R$ in Fact~\ref{fact:slates} in terms of $B$
and $L$. This is the regret bound one would obtain by using the
approach of \citet{swaminathan2016off} in our precise setting,
ignoring the semibandit feedback, but with known weight-vector bound
$B$.  The second comparison finds a conservative value of $B$ in
terms of $R$ and $L$.

For the first comparison, recall that our setting makes no assumptions on the scale of the reward, except that the noise $\xi$ is bounded in $[-1,1]$, so
our setting never admits $R < 1$.
If we begin with a setting of $B$, we need to conservatively set $R = \max\{B\sqrt{L}, 1\}$, which gives the dependence
\begin{align*}
 &\text{EELS: }\max\{B^{1/3}L^{1/2}, BL^{1/6}\}\\
 &\text{\citet{swaminathan2016off}: } \max\{BL^{5/6}, L^{1/3}\}.
\end{align*}
The EELS bound is never worse than the bound in Fact~\ref{fact:slates}
according to this comparison.  At $B=\Theta(L^{-1/2})$, the two bounds
are of the same order, which is $\Theta(L^{1/3})$.  For $B =
\order(L^{-1/2})$, the EELS bound is at most $L^{1/3}$, while for $B =
\Omega(L^{-1/2})$ the first term in the EELS bound is at most the
first term in the \citet{swaminathan2016off} bound.  In both cases,
the EELS bound is superior.  Finally when $B = \Omega(\sqrt{L})$, the
second term dominates our bound, so EELS demonstrates an $L^{2/3}$
improvement.

For the second comparison, since our setting has the noise bounded in
$[-1,1]$, assume that $R \ge 1$ and that the total reward is scaled in
$[-R,R]$ as in Fact~\ref{fact:slates}.  If we want to allow any $y(A)
\in [0,1]^L$, the tightest setting of $R$ is between $\norm{\ws}_1/2$
and $\norm{\ws}_1$ (depending on the structure of the positive and
negative coordinates of $\ws$).  For simplicity, assume $R$ is a bound
on $\norm{\ws}_1$. Since the EELS bound depends on $B$, a bound on the
Euclidean norm of $\ws$, we use $\norm{\ws}_2 \le \norm{\ws}_1 \le
\sqrt{L}\norm{\ws}_2$ to obtain a conservative setting of $B=R$.
This gives the dependence
\begin{align*}
 &\text{EELS: }\max\{R^{1/3}L^{1/2}, RL^{1/6}\}\\
 &\text{\citet{swaminathan2016off}: }RL^{1/3}
\end{align*}
Since $R\ge 1$, the EELS bound is superior whenever $R\ge
L^{1/4}$. Moreover, if $R=\Omega({\sqrt{L}})$, i.e., at least
$\sqrt{L}$ positions are relevant, the second term dominates our
bound, and we improve by a factor of $L^{1/6}$.  The EELS bound is
inferior when $R \le L^{1/4}$, which corresponds to a high-sparsity
case since $R$ is also a bound on $\norm{\ws}_1$ in this comparison.

\section{Implementation Details}
\label{sec:app_exp}

\subsection{Implementation of \semimonster}
\semimonster is implemented as stated in Algorithm~\ref{alg:semimonster} with some modifications, primarily to account for an imperfect oracle.
OP is solved using the coordinate descent procedure described in Appendix~\ref{sec:optimization_full}.

We set $\psi = 1$ in our implementation and ignore the log factor
in $\mu_t$.  Instead, since $\pmin=L$, we use
$\mu_t = c\sqrt{1/KLT}$ and tune $c$, which can
compensate for the absence of the $\log(t^2N/\delta)$ factor.  This
additionally means that we ignore the failure probability parameter
$\delta$.  Otherwise, all other parameters and constants are set as
described in Algorithm~\ref{alg:semimonster} and OP.

As mentioned in Section~\ref{sec:experiments}, we implement \amo via a
reduction to squared loss regression.  There are many possibilities
for this reduction. In our case, we specify a squared loss regression problem
via a dataset $D = \{x_i,
A_i, y_i, \gamma_i\}_{i=1}^n$ where $x_i \in \Xcal$, $A_i$ is any list of
actions, $y_i \in \RR^K$ assigns a value to each action, and $\gamma_i
\in \RR^K$ assigns an importance weight to each action.  Since in our
experiments $\ws = \mathbf{1}$, we do not need to pass along the vectors
$v_i \in \RR^L$ described in Eq.~\eqref{eqn:amo}.

Given such a dataset $D$, we minimize a \emph{weighted squared loss} objective over a regression class $\Fcal$,
\begin{align}
\hat{f} = \argmin_{f \in \Fcal} \sum_{i=1}^n \sum_{a \in A_i}
\gamma_i(a) (f(\phi(x_i,a)) - y_i(a))^2,
\label{eqn:amo_implementation}
\end{align}
where $\phi(x,a)$ is a feature vector associated with the given query-document pair.
Note that we only include terms corresponding to simple actions in $A_i$ for each $i$.
This regression function is associated with the greedy policy that chooses the
best valid ranking according to the sum of rewards of individual actions as predicted by $\hat{f}$
on the current context.

We access this oracle with two different kinds of datasets.  When we
access \amo to find the empirically best policy, we only use the history
of the interaction.  In this case, we only regress onto the chosen
actions in the history and we let $\gamma_i$ be their importance
weights.  More formally, suppose that at round $t$, we observe context
$x_t$, choose composite action $A_t\sim q_t$ and
receive feedback $\{y_t(a_{t,\ell})\}_{\ell=1}^L$.  We create a single
example $(x_t,A_t,z_t,\gamma_t)$ where $x_t$ is the context, $A_t$ is
the chosen composite action, $z_t$ has $z_t(a) = \mathbf{1}(a \in
  A_t)y_t(a)$ and $\gamma_t(a) = 1/q_t(a\in A_t)$.  Observe that
when this sample is passed into Eq.~\eqref{eqn:amo_implementation}, it
leads to a different objective than if we regressed directly onto the
importance-weighted reward features $\hat{y}_t$.

We also create datasets to verify the variance constraint within OP.
For this, we use the \amo in a more direct way by setting $A_t$ to be a list of all $K$ actions, letting $y_t$ be the importance weighted vector, and $\gamma_t= \mathbf{1}$.

We use this particular implementation because leaving the importance weights inside the square loss term introduces additional variance, which we would like to avoid.

The imperfect oracle introduces one issue that needs to be corrected.
Since the oracle is not guaranteed to find the maximizing policy on every dataset, in the $t$th round of the algorithm,
we may encounter a policy $\pi$ that has $\Reghat_t(\pi) < 0$, which can cause the coordinate descent procedure to loop indefinitely.
Of course, if we ever find a policy $\pi$ with $\Reghat_t(\pi) < 0$, it means that we have found a better policy, so we simply switch the leader.
We found that with this intuitive change, the coordinate descent procedure always terminates in a few iterations.

\subsection{Implementation of \epsgreedy}
Recall that we run a variant of \epsgreedy where at each round we explore with probability $\epsilon$ and exploit with probability $(1-\epsilon)$, which is slightly different from the explore-first algorithm analyzed in Appendix~\ref{app:eps_greedy}.

For \epsgreedy, we also use the oracle defined in Eq.~\eqref{eqn:amo_implementation}.
This algorithm only accesses the oracle to find the empirically best policy, and we do this in the same way as \semimonster does, i.e.,
we only regress onto actions that were actually selected with importance weights encoded via $\gamma_i$s.
We use all of the data, including the data from exploitation rounds, with importance weighting.


\subsection{Implementation of \linucb}
The semibandit version of \linucb uses ridge regression to predict the semibandit feedback given query-document features $\phi(x,a)$.
If the feature vectors are in $d$ dimensions, we start with $\Sigma_1 = I_d$ and $\theta_1 = 0$, the all zeros vector.
At round $t$, we receive the query-document feature vectors $\{\phi(x_t, a)\}_{a \in \Acal}$ for query $x_t$ and we choose
\begin{align*}
A_t = \argmax_{A} \left\{\sum_{a \in A} \theta_t^T\phi(x_t,a) + \alpha \phi(x_t,a)^T\Sigma_t^{-1}\phi(x_t,a)\right\}.
\end{align*}
Since in our experiment we know that $\ws = \mathbf{1}$ and all rankings are valid,
the order of the documents is irrelevant and the best ranking consists of the top $L$ simple actions with the
largest values of the above ``regularized score''.
Here $\alpha$ is a parameter of the algorithm that we tune.

After selecting a ranking, we collect the semibandit feedback $\{y_t(a_{t,\ell})\}_{\ell=1}^L$.
The standard implementation would perform the update
\begin{align*}
\Sigma_{t+1} &\gets \Sigma_t + \sum_{\ell=1}^L \phi(x_t,a_{t,\ell})\phi(x_t,a_{t,\ell})^T, \qquad 
\theta_{t+1} \gets \Sigma_{t+1}^{-1}\left(\sum_{i=1}^t \sum_{\ell=1}^L \phi(x_i,a_{i,\ell}) y_i(a_{i,\ell})\right),
\end{align*}
which is the standard online ridge regression update.
For computational reasons, we only update every 100 iterations, using all of the data.
Thus, if $\textrm{mod}(t,100) \ne 0$, we set $\Sigma_{t+1} \gets \Sigma_t$ and $\theta_{t+1} \gets \theta_t$.
If $\textrm{mod}(t,100) = 0$, we set
\begin{align*}
\Sigma_{t+1} &\gets I + \sum_{i=1}^t\sum_{\ell=1}^L \phi(x_i,a_{i,\ell})\phi(x_i,a_{i,\ell})^T, \qquad 
\theta_{t+1} \gets \Sigma_{t+1}^{-1}\left(\sum_{i=1}^t \sum_{\ell=1}^L \phi(x_i,a_{i,\ell}) y_i(a_{i,\ell})\right).
\end{align*}

\subsection{Policy Classes}
As \amo for both \semimonster and \epsgreedy, we use the default implementations of regression
with various function classes in \sklearn version 0.17. 
We instantiate \sklearn model objects and use the \texttt{fit()} and \texttt{predict()} routines.
The model objects we use are
\begin{enums}
\item \texttt{sklearn.linear\_model.LinearRegression()}
\item \texttt{sklearn.ensemble.GradientBoostingRegressor(n\_estimators=50,max\_depth=2)}
\item \texttt{sklearn.ensemble.GradientBoostingRegressor(n\_estimators=50,max\_depth=5)}
\end{enums}
All three objects accommodate weighted least-squares objectives as required by Eq.~\eqref{eqn:amo_implementation}.

\section{Proof of Regret Bound in Theorem~\ref{thm:semimonster}}
\label{sec:monster_full}

The proof hinges on two uniform deviation bounds, and then a careful inductive analysis of the regret using the OP.
We only need our two deviation bounds to hold for the rounds $t$ in which
$\mu_t=\sqrt{\ln(16t^2N/\delta)/(Kt\pmin)}$. Let $d_t \coloneqq \ln(16t^2N/\delta)$. These rounds then
start at
\[
  t_0 \coloneqq
  \min\Braces{t:\: \sqrt{\frac{d_t}{Kt\pmin}} \le \frac{1}{2K}}
  =
  \min\Braces{t:\: \frac{d_t}{t} \le \frac{\pmin}{4K}}
.
\]
Note that $t_0 \ge 4$ since $d_t \ge 1$ and $K \ge \pmin$. From the definition of $t_0$, we have for all $t\ge t_0$:
\begin{equation}
\label{eq:t0:bounds}
  \mu_t\ge\sqrt{d_t/(Kt\pmin)},
\quad
  t\ge 4Kd_t/\pmin.
\end{equation}

The first deviation bound shows that the variance estimates used in \Eq{op_variance} are suitable estimators for the true variance of the distribution.
To state this deviation bound, we need some definitions:
\begin{align}
V(P,\pi,\mu) \coloneqq \EE_{x \sim \Dcal_x}\left[\sum_{\ell=1}^L \frac{1}{P^\mu(\pi(x)_\ell \given x)}\right],
\quad
\hat{V}_t(P,\pi,\mu) \coloneqq \hat{\EE}_{x \sim H_t}\left[\sum_{\ell=1}^L \frac{1}{P^\mu(\pi(x)_\ell \given x)}\right].
\end{align}
In these definitions and throughout this appendix we use the shorthand $P(a\given x)$ to mean $P(a\in A\given x)$ for any projected
subdistribution $P(A\given x)$. If $P$ is a distribution, we have $\sum_{a\in\Acal}P(a\given x)=L$. For
a subdistribution, this sum can be smaller, so $\sum_{a \in \Acal}P(a\given x) \le L$ for all subdistributions. 
The deviation bound is in the following theorem:
\begin{theorem}
\label{thm:variance_deviation}
Let $\delta \in (0,1)$. Then with probability at least $1-\delta/8$,
for all $t\ge t_0$, all distributions $P$ over $\Pi$, and all $\pi \in \Pi$, we have
\begin{align}
V(P,\pi,\mu_t) \le 6.4 \hat{V}_t(P,\pi,\mu_t) + 81.3\frac{KL}{\pmin}.
\label{eq:variance_deviation}
\end{align}
\end{theorem}
\begin{proof}
The proof of this theorem is similar to a related result of~\citet{agarwal2014taming} (See their Lemma 10). 
We first use Freedman's inequality (Lemma~\ref{lem:freedman}) to argue that for a fixed $P, \pi, \mu$, and $t$, the empirical version of the variance is close to the true variance.
We then use a discretization of the set of all distributions and take a union bound to extend this deviation inequality to all $P, \pi, \mu, t$.

To start, we have:
\begin{lemma}
\label{lem:one_var_dev}
For fixed $P, \pi, \mu, t$ and for any $\lambda \in \left[0, \frac{\mu \pmin}{L}\right]$, with probability at least $1-\delta$:
\begin{align*}
V(P, \pi, \mu) - \hat{V}_t(P, \pi, \mu) \le \frac{(e-2)\lambda L}{\mu\pmin} V(P, \pi, \mu) + \frac{\ln(1/\delta)}{t \lambda}
\end{align*}
\end{lemma}
\begin{proof}
Let:
\begin{align*}
Z_i \coloneqq \sum_{\ell=1}^L \frac{1}{P^\mu(\pi(x_i)_\ell | x_i)} - \EE_{x \sim \Dcal_x} \sum_{\ell=1}^L \frac{1}{P^\mu(\pi(x)_\ell | x)},
\end{align*}
and notice that $\frac{1}{t}\sum_{i=1}^t Z_i = \hat{V}_t(P, \pi, \mu) - V(P, \pi,\mu)$. Clearly, $\EE Z_i = 0$ for all $i$ and $\max_i |Z_i| \le L/\mu \pmin$ since when we smooth by $\mu$, each simple action that $\pi$ could choose must appear with probability at least $\mu \pmin$.
By the Cauchy-Schwarz and Holder inequalities, the conditional variance is:
\begin{align*}
\EE_{x \sim \Dcal_x} Z_i^2 &\le \EE_{x \sim \Dcal_x} \left(\sum_{\ell=1}^L \frac{1}{P^\mu(\pi(x)_\ell | x)}\right)^2 \le L \EE_{x \sim \Dcal_x}\sum_{\ell=1}^L \frac{1}{P^\mu(\pi(x)_\ell|x)^2}\\
& \le \frac{L}{\mu \pmin} \EE_{x \sim \Dcal_x} \sum_{\ell=1}^L \frac{1}{P^\mu(\pi(x)_\ell|x)} = \frac{L}{\mu \pmin} V(P, \pi, \mu).
\end{align*}
The lemma now follows by Freedman's inequality.
\end{proof}

To prove the variance deviation bound of \Thm{variance_deviation}, we next use a discretization lemma from~\cite{dudik2011efficient} (their Lemma 16) which immediately implies that for any $P$, there exists a distribution $P'$ supported on at most $N_t$ policies such that for $c_t > 0$, if $N_t \ge \frac{6}{\gamma_t^2 \mu_t \pmin}$:
\begin{align*}
V(P, \pi, \mu) - V(P', \pi, \mu_t) + c_t\left(\hat{V}_t(P', \pi, \mu_t) - \hat{V}_t(P, \pi, \mu_t)\right) \le \gamma_t(V(P,\pi, \mu_t) + c_t\hat{V}_t(P, \pi, \mu_t))
\end{align*}
This is exactly the second conclusion of their Lemma 16 except we use $c_t$ instead of their $(1+\lambda)$ (we will set $c_t > 1$). 
The other difference is the inclusion of $\pmin$ in the lower bound on $N_t$, which is based on a straightforward modification to their proof. 

We set $\gamma_t = \sqrt{\frac{1-K\mu_t}{N_t\mu_t\pmin}} + 3\frac{1-K\mu_t}{N_t\mu_t\pmin}, c_t = \frac{1}{1 - \frac{(e-2)L\lambda_t}{\mu_t \pmin}}, N_t = \lceil \frac{12(1-K\mu_t)}{\mu_t\pmin}\rceil$ and $\lambda_t = 0.66\mu_t\pmin/L$.
The choice of $c_t$ is motivated by Lemma~\ref{lem:one_var_dev}, which can be rearranged to (for a distribution $P'$)
\begin{align*}
  &\ V(P',\pi,\mu_t) - \frac{1}{1-\frac{(e-2)L \lambda_t }{\mu_t \pmin}} \hat{V}_t(P',\pi,\mu_t) \le \frac{1}{1-\frac{(e-2)L \lambda_t}{\mu_t \pmin}} \frac{\ln(1/\delta)}{t\lambda_t}\\
  \Leftrightarrow \ & V(P',\pi,\mu_t) - c_t \hat{V}_t(P',\pi,\mu_t) \le c_t \frac{\ln(1/\delta)}{t\lambda_t}.
\end{align*}
To take a union over all $t \in \NN$, $N_t$-point distributions $P$ over $\Pi$, and all $\pi \in \Pi$, we set $\delta_t = \delta (\frac{1}{2t^2 N^{N_t+1}})$ in the $t$th iteration.
This inequality becomes
\begin{align*}
V(P',\pi,\mu_t) - c_t \hat{V}_t(P',\pi,\mu_t) \le c_t \frac{\ln(2N^{N_t+1}t^2/\delta)}{t\lambda_t}.
\end{align*}
The choice of $c_t$ and $\lambda_t$ leads to a bound $c_t = \frac{1}{1-0.66(e-2)} \le 1.91$.

We also use the values of $N_t$ and $\gamma_t$ to bound
\begin{align*}
  \gamma_t = \sqrt{\frac{1-K\mu_t}{N_t\mu_t\pmin}} + 3\frac{1-K\mu_t}{N_t\mu_t\pmin} \le \sqrt{\frac{1}{12}} + \frac{1}{4}.
\end{align*}

Rearranging the discretization claim gives
\begin{align*}
  V(P,\pi,\mu_t) &\le \frac{c_t(1+\gamma_t)}{(1-\gamma_t)}\hat{V}_t(P,\pi,\mu_t) + \frac{1}{(1-\gamma_t)}\left(V(P',\pi,\mu_t) - c_t \hat{V}_t(P',\pi,\mu_t)\right)\\
  & \le 6.4\hat{V}_t(P,\pi,\mu_t) +\frac{c_t}{(1-\gamma_t)}\frac{\ln(2N^{N_t+1}t^2/\delta)}{t\lambda_t}.
\end{align*}
Using the bounds on $c_t,\gamma_t$ and the settings of $N_t$ and $\lambda_t$, this last term is at most
\begin{align*}
& \frac{c_t}{(1-\gamma_t)}\left(\frac{L \ln(2N^2t^2/\delta)}{t\mu_t\pmin} + \frac{LN_t\ln(N)}{t\mu_t\pmin}\right)
\le \frac{6.3 L\ln(16N^2t^2/\delta)}{\mu_t t \pmin} + \frac{75L(1-K\mu_t)\ln(N)}{\mu_t^2t\pmin^2}.
\end{align*}

The theorem now follows from the bounds of \Eq{t0:bounds}.
\end{proof}

The other main deviation bound is a straightforward application of Freedman's inequality and a union bound.
To state the lemma, we must introduce one more definition.
Let
\[
  \Vcal_t(\pi) \coloneqq \max_{0 \le \tau \le t-1} V(\tilde{Q}_\tau, \pi, \mu_\tau)
\]
where $\tilde{Q}_\tau$ is the distribution calculated in Step~\ref{step:tildeQ} of Algorithm~\ref{alg:semimonster}. Note that $\tilde{Q}^{\mu_\tau}_\tau$ is the distribution used to select the composite action in round $\tau+1$.
\begin{lemma}
\label{lem:reward_deviation}
Let $\delta \in (0,1)$. Then with probability at least $1-\delta/4$,
for all $t\ge t_0$ and
$\pi \in \Pi$, we have
\begin{align}
|\eta_t(\pi,\ws) - \Rcal(\pi)| \le
\normratio\Vcal_t(\pi)\pmin\mu_t + \normratio KL\mu_t.
\end{align}
\end{lemma}
\begin{proof}
Consider a specific $t\ge t_0$ and $\pi\in\Pi$. Let
\[
  Z_i\coloneqq\angles{\ws,\hat{y}_i(\pi(x_i))}-\angles{\ws,y_i(\pi(x_i))}
\]
and note that $\frac1t\sum_{i=1}^t Z_i=\eta_t(\pi,\ws) - \Rcal(\pi)$. Since $\hat{y}_i$ is an unbiased estimate of $y_i$, the $Z_i$s form a martingale. The range of each $Z_i$ is bounded
as
\[
  \abs{Z_i}\le\norm{\ws}_1\norm{\hat{y}_i-y_i}_\infty\le\frac{\norm{\ws}_1}{\mu_{i-1}\pmin}\le\frac{\norm{\ws}_1}{\mu_t\pmin},
\]
because the $\mu_i$s are non-increasing.
The conditional variance can be bounded via the Cauchy-Schwarz inequality:
\begin{align*}
\EE[Z_i^2 \given H_{i-1}]
& \le
\|\ws\|_2^2\sum_{\ell=1}^L\EE_{x \sim \Dcal_x}\EE_{y\given x} \frac{y(\pi(x)_\ell)^2}{\tilde{Q}^{\mu_{i-1}}_{i-1}(\pi(x)_\ell\given x)}
\\
& \le
\|\ws\|_2^2V(\tilde{Q}_{i-1}, \pi, \mu_{i-1}) \le \|\ws\|_2^2 \Vcal_t(\pi).
\end{align*}
By Freedman's inequality with $\lambda=\mu_t\pmin/\norm{\ws}_1$, we have, with probability at least $1-\delta/(8t^2N)$,
\begin{align}
\notag
|\eta_t(\pi,\ws) - \Rcal(\pi)|
&\le
\frac{\mu_t\pmin}{\norm{\ws}_1}\cdot\norm{\ws}_2^2\Vcal_t(\pi) + \frac{d_t}{t}\cdot\frac{\norm{\ws}_1}{\mu_t\pmin}
\\
\label{eq:by:t0}
&\le
\normratio\Vcal_t(\pi)\pmin\mu_t + K\mu_t\norm{\ws}_1
\\
\label{eq:by:L1:L2}
&\le
\normratio\Vcal_t(\pi)\pmin\mu_t + \normratio KL\mu_t.
\end{align}
Here, \Eq{by:t0} follows because $d_t/(\pmin t) \le K\mu_t^2$ by \Eq{t0:bounds}.
\Eq{by:L1:L2} follows because $\norm{\ws}_1\le L \norm{\ws}_2^2/\norm{\ws}_1$ by the fact that $\norm{\ws}_1\le\sqrt{L}\norm{\ws}_2$.
The lemma follows by a union bound over all $t\ge t_0$ and $\pi\in\Pi$.
\end{proof}

Equipped with these two deviation bounds we will proceed to prove the main theorem.
Let $\Ecal$ denote the event that both the variance and reward deviation bounds of \Thm{variance_deviation} and \Lem{reward_deviation} hold. Note that $\PP(\Ecal) \ge 1-\delta/2$.
Using the variance constraint, it is straightforward to prove the following lemma:
\begin{lemma}
\label{lem:var_control}
Assume event $\Ecal$ holds, then for any round $t\ge 1$ and any policy $\pi \in \Pi$, let $t_\star$ be the round achieving the $\max$ in the definition of $\Vcal_t(\pi)$.
Then there are universal constants $\theta_1 \ge 2$ and $\theta_2$ such that:
\begin{align}
\Vcal_t(\pi) \le
\begin{cases}
\displaystyle
  \frac{2KL}{\pmin}
  & \text{if $t_\star < t_0$,}
\\[12pt]
\displaystyle
  \frac{\theta_1 KL}{\pmin} + \invnormratio \frac{\Reghat_{t_\star}(\pi)}{\theta_2\pmin\mu_{t_\star}}
  & \text{if $t_\star \ge t_0$.}
\end{cases}
\end{align}
\end{lemma}
\begin{proof}
The first claim follows by the definition of $\Vcal_t(\pi)$ and the fact that $\mu_\tau=1/2K$ for $\tau<t_0$.
For the second claim, we use the variance deviation bound and the optimization constraint.
In particular, since $t_\star\ge t_0$, we can apply Theorem~\ref{thm:variance_deviation}:
\begin{align*}
V(\tilde{Q}_{t_\star}, \pi, \mu_{t_\star}) \le 6.4 \hat{V}_{t_\star}(\tilde{Q}_{t_\star}, \pi, \mu_{t_\star}) + 81.3\frac{KL}{\pmin},
\end{align*}
and we can use the optimization constraint which gives an upper bound on $\hat{V}_{t_\star}(\tilde{Q}_{t_\star}, \pi, \mu_{t_\star})$:
\begin{align*}
\hat{V}_{t_\star}(\tilde{Q}_{t_\star}, \pi, \mu_{t_\star}) \le \hat{V}_{t_\star}(Q_{t_\star}, \pi, \mu_{t_\star}) \le \frac{2KL}{\pmin} + \invnormratio \frac{\Reghat_{t_\star}(\pi)}{\psi\pmin\mu_{t_\star}}
\end{align*}
The bound follows by the choice $\theta_1=94.1$ and $\theta_2 = \psi/6.4$.
\end{proof}

We next compare $\Reg(\pi)$ and $\Reghat(\pi)$ using the variance bounds above.
\begin{lemma}
\label{lem:regret_induction}
Assume event $\Ecal$ holds and define $c_0 \coloneqq 4(1+\theta_1)$.
For all $t \ge t_0$ and all policies $\pi \in \Pi$:
\begin{align}
\Reg(\pi) \le 2 \Reghat_t(\pi) + c_0 \normratio KL\mu_t \quad \mbox{and} \quad \Reghat_t(\pi) \le 2 \Reg(\pi) + c_0 \normratio KL\mu_t.
\end{align}
\end{lemma}
\begin{proof}
The proof is by induction on $t$.
As the base case, consider $t = t_0$ where we have $\mu_\tau = 1/(2K)$ for all $\tau< t_0$, so $\Vcal_t(\pi) \le 2KL/\pmin$ for all $\pi \in \Pi$ by Lemma~\ref{lem:var_control}.
Using the reward deviation bound of \Lem{reward_deviation}, which holds under $\Ecal$, we thus have
\begin{align*}
|\eta_t(\pi,\ws) - \Rcal(\pi)| \le
\normratio\Vcal_t(\pi)\pmin\mu_t + \normratio KL\mu_t
\le 3\normratio KL\mu_t
\end{align*}
for all $\pi \in \Pi$.
Now both directions of the bound follow from the triangle inequality and the optimality of $\pi_t$ for $\eta_t(\cdot)$ and $\pi_\star$ for $\Rcal(\cdot)$,
using the fact that $c_0 \ge 6$ from the definition of $\theta_1$.

For the inductive step, fix some round $t$ and assume that the claim holds for all $t_0 \le t' < t$ and all $\pi \in \Pi$.
By the optimality of $\pi_t$ for $\eta_t$ and Lemma~\ref{lem:reward_deviation}, we have
\begin{align*}
\Reg(\pi) - \Reghat_t(\pi) &= (\Rcal(\pi_\star) - \Rcal(\pi)) - (\eta_t(\pi_t,\ws) - \eta_t(\pi,\ws))\\
& \le (\Rcal(\pi_\star) - \Rcal(\pi)) - (\eta_t(\pi_\star,\ws) - \eta_t(\pi,\ws))\\
& \le (\Vcal_t(\pi_\star)+ \Vcal_t(\pi))\normratio \pmin\mu_t + 2\normratio KL\mu_t.
\end{align*}
Now by Lemma~\ref{lem:var_control}, there exist rounds $i,j < t$ such that
\begin{align*}
\Vcal_t(\pi) &\le \frac{\theta_1 KL}{\pmin} + \invnormratio \frac{\Reghat_i(\pi)}{\theta_2\pmin\mu_i} \mathbf{1}(i\ge t_0)\\
\Vcal_t(\pi_\star) &\le \frac{\theta_1 KL}{\pmin} + \invnormratio \frac{\Reghat_j(\pi_\star)}{\theta_2\pmin\mu_j} \mathbf{1}(j\ge t_0)
\end{align*}
For the term involving $\Vcal_t(\pi)$, if $i<t_0$, we immediately have the bound
\begin{align*}
\Vcal_t(\pi)\normratio \pmin\mu_t \le \theta_1 \normratio KL \mu_t.
\end{align*}
On the other hand, if $i\ge t_0$ then using the fact that $\mu_i\ge\mu_t$,
and applying the inductive hypothesis to $\Reghat_i(\pi)$ gives:
\[
\Vcal_t(\pi) \normratio \pmin\mu_t
\le \theta_1\normratio KL\mu_t + \frac{\Reghat_i(\pi)\mu_t}{\theta_2\mu_i}
\le \Parens{\theta_1+\frac{c_0}{\theta_2}}\normratio KL\mu_t + \frac{2 \Reg(\pi)}{\theta_2}.
\]
Similarly for the $\Vcal_t(\pi_\star)$ term, we have the bound
\begin{align*}
\Vcal_t(\pi_\star) \normratio \pmin\mu_t \le \left(\theta_1 + \frac{c_0}{\theta_2}\right) \normratio KL\mu_t + \frac{2\Reg(\pi_\star)}{\theta_2} = \left(\theta_1 + \frac{c_0}{\theta_2}\right) \normratio KL\mu_t,
\end{align*}
since $\pi_\star$ has no regret.
Combining these bounds gives:
\[
\Reg(\pi) - \Reghat_t(\pi)
\le 2\Parens{\theta_1+\frac{c_0}{\theta_2}}\normratio KL\mu_t + \frac{2 \Reg(\pi)}{\theta_2} + 2\normratio KL\mu_t,
\]
which gives
\[
\Reg(\pi) \le \frac{1}{1-2/\theta_2}\left(\Reghat_t(\pi) + 2\left(1+\theta_1 + \frac{c_0}{\theta_2}\right) \normratio KL \mu_{t-1} \right).
\]
Recall that $\theta_1 = 94.1, \theta_2 = \psi/6.4, \psi = 100$, and $c_0 = 4(1+\theta_1)$.
This means that $\theta_2>15.6$, so $2/\theta_2 \le 1/2$, and hence the pre-multiplier on the $\Reghat_t(\pi)$ term is at most $2$.
To finish proving the bound on $\Reg(\pi)$, it remains to show that $c_0\ge 2(1+\theta_1+ c_0/\theta_2)/(1-2/\theta_2)$, or equivalently,
that
\[
 c_0\Parens{1-4/\theta_2}\ge 2\Parens{1+\theta_1}.
\]
This holds, because $c_0(1-4/\theta_2)=4(1-4/\theta_2)(1+\theta_1)$ and $4/\theta_2\le 1/2$.

The other direction proceeds similarly.
Under event $\Ecal$ we have:
\begin{align*}
\Reghat_t(\pi) - \Reg(\pi) &= \eta_t(\pi_t,\ws) - \eta_t(\pi,\ws) - \Rcal(\pi_\star) + \Rcal(\pi)\\
& \le \eta_t(\pi_t,\ws) - \eta_t(\pi,\ws) - \Rcal(\pi_t) + \Rcal(\pi)\\
& \le (\Vcal_t(\pi) + \Vcal_t(\pi_t))\normratio \pmin\mu_t + 2\normratio KL\mu_t.
\end{align*}
As before, we have the bound:
\begin{align*}
\Vcal_t(\pi)\normratio \pmin\mu_t &\le \left(\theta_1 + \frac{c_0}{\theta_2}\right)\normratio KL\mu_t + \frac{2 \Reg(\pi)}{\theta_2},
\end{align*}
but for the $\Vcal_t(\pi_t)$ term we must use the inductive hypothesis twice.
We know there exists a round $j < t$ for which
\begin{align*}
\Vcal_t(\pi_t) \le \theta_1 \frac{KL}{\pmin} + \invnormratio \frac{\Reghat_j(\pi)}{\theta_2\pmin\mu_j} \mathbf{1}(j\ge t_0).
\end{align*}
Applying the inductive hypothesis twice gives:
\begin{align*}
\invnormratio \frac{\Reghat_j(\pi_t)}{\theta_2 \pmin\mu_j} &\le \invnormratio\frac{\left(2\Reg(\pi_t) + c_0 \normratio KL \mu_j\right)}{\theta_2 \pmin\mu_j}\\
&\le \invnormratio\frac{2\left(2\Reghat_t(\pi_t) + c_0 \normratio KL \mu_t\right) + c_0 \normratio KL\mu_j}{\theta_2 \pmin\mu_j}\\
& \le \frac{3 c_0}{\theta_2} \frac{KL}{\pmin}.
\end{align*}
Here we use the inductive hypothesis twice, once at round $j$ and once at round $t$, and then use the fact that $\pi_t$ has no regret at round $t$, i.e., $\Reghat_t(\pi_t) = 0$.
We also use the fact that the $\mu_t$s are non-increasing, so $\mu_t/\mu_j \le 1$.
This gives the bound:
\begin{align*}
\Vcal_t(\pi_t) \normratio \pmin\mu_t \le \left(\theta_1 + \frac{3c_0}{\theta_2}\right)\normratio KL \mu_t.
\end{align*}
Combining the bounds for $\Vcal_t(\pi)$ and $\Vcal_t(\pi_t)$ gives:
\begin{align*}
\Reghat_t(\pi) \le \left(1+\frac{2}{\theta_2}\right)\Reg(\pi) + \left(2\theta_1 + \frac{4c_0}{\theta_2}+2\right) \normratio KL \mu_t.
\end{align*}
Since $\theta_2 \ge 2$, the pre-multiplier on the first term is at most $2$.
It remains to show that $c_0\ge 2(1+\theta_1)+4c_0/\theta_2$. This is again equivalent to $c_0(1-4/\theta_2)\ge 2(1+\theta_1)$, which
holds as before.
\end{proof}

The last key ingredient of the proof is the following lemma, which shows that the low-regret constraint in \Eq{op_regret}, based on the regret estimates, actually ensures low regret.
\begin{lemma}
\label{lem:reg_constraint_use}
Assume event $\Ecal$ holds.
Then for every round $t\ge 1$:
\begin{align}
\sum_{\pi \in \Pi} \tilde{Q}_{t-1}(\pi)\Reg(\pi) \le (4\psi + c_0) \normratio KL\mu_{t-1}
\end{align}
\end{lemma}
\begin{proof}
If $t \le t_0$ then $\mu_{t-1} = 1/(2K)$ in which case (since $\Reg(\pi) \le \|\ws\|_1$):
\begin{align*}
\sum_{\pi \in \Pi}\tilde{Q}_{t-1}(\pi)\Reg(\pi) \le \|\ws\|_1 \le \normratio L = 2 \normratio KL \mu_{t-1}
\le (4\psi+c_0)\normratio KL \mu_{t-1}.
\end{align*}
For $t > t_0$, we have:
\begin{align*}
\sum_{\pi \in \Pi} \tilde{Q}_{t-1}(\pi) \Reg(\pi) &\le \sum_{\pi \in \Pi}\tilde{Q}_{t-1}(\pi)\left(2 \Reghat_{t-1}(\pi) + c_0 \normratio KL \mu_{t-1}\right)\\
& \le \left(2 \sum_{\pi \in \Pi}Q_{t-1}(\pi)\Reghat_{t-1}(\pi)\right) + c_0 \normratio KL\mu_{t-1}\\
& \le (4\psi+c_0) \normratio KL \mu_{t-1}.
\end{align*}
The first inequality follows by Lemma~\ref{lem:regret_induction} and the second follows from the fact that $\tilde{Q}_{t-1}$ places its remaining mass (compared with $Q_{t-1}$)
on $\pi_{t-1}$ which suffers no empirical regret at round $t-1$.
The last inequality is due to the low-regret constraint in the optimization.
\end{proof}

To control the regret, we must first add up the $\mu_t$s, which relate to the exploration probability:
\begin{lemma}
\label{lem:mu_sum}
For any $T\ge 1$:
\begin{align*}
  \sum_{t=1}^T\mu_{t-1} \le 2 \sqrt{\frac{Td_T}{K\pmin}}.
\end{align*}
\end{lemma}
\begin{proof}
We will use the identity
\begin{equation}
\label{eq:bnd}
  \frac{1}{K}\le
\sqrt{\frac{d_T}{K\pmin}},
\end{equation}
which holds, because $d_T\ge 1$ and $K\ge\pmin$. We prove the lemma separately for $T=1$ and $T\ge 2$.
Since $t_0\ge 4$, we have $\mu_0=1/2K$. Thus, for $T=1$, by \Eq{bnd}:
\[
  \sum_{t=1}^T\mu_{t-1}
  =
  \frac{1}{2K}
  \le
  \frac{1}{2}\sqrt{\frac{d_T}{K\pmin}}
  \le
  2\sqrt{\frac{Td_T}{K\pmin}}.
\]

For $T\ge 2$, we use the fact that $\mu_0=\mu_1=1/2K$, and $\mu_t\le\sqrt{d_T/(Kt\pmin)}$ for $t\le T$:
\begin{align}
\notag
  \sum_{t=1}^T\mu_{t-1}
  &\le
  \frac{1}{K}
  +\sqrt{\frac{d_T}{K\pmin}}\sum_{t=3}^T\frac{1}{\sqrt{t-1}}
\\
\label{eq:telescope}
  &\le
  \sqrt{\frac{d_T}{K\pmin}}
  +\sqrt{\frac{d_T}{K\pmin}}\BigParens{2\sqrt{T-1}-2}
\\
\notag
  &\le
  2\sqrt{\frac{Td_T}{K\pmin}}.
\end{align}
In \Eq{telescope}, we bounded the first term using \Eq{bnd} and the second term using the telescoping
identity $1/\sqrt{t-1}\le 2\sqrt{t-1}-2\sqrt{t-2}$, which holds for $t\ge 2$.
\end{proof}

We are finally ready to prove the theorem by adding up the total regret for the algorithm.
\begin{lemma}
\label{lem:monster_regret_bound}
For any $T \in \NN$, with probability at least $1-\delta$, the regret after $T$ rounds is at most:
\begin{align*}
\normratio L \left[2\sqrt{2T \ln(2/\delta)} + 2(4\psi + c_0 + 1)\sqrt{\frac{KTd_T}{\pmin}}\right].
\end{align*}
\end{lemma}
\begin{proof}
For each round $t\ge 1$, let $Z_t \coloneqq r_t(\pi_\star(x_t)) - r_t(A_t) -
\sum_{\pi \in \Pi}\tilde{Q}_{t-1}^{\mu_{t-1}}(\pi)\Reg(\pi)$.  Since at
round $t$, we play action $A_t$ with probability
$\tilde{Q}_{t-1}^{\mu_{t-1}}(A_t)$, we have $\EE Z_t=0$. Moreover, since the noise term $\xi$ is shared between
$r_t(\pi_\star(x_t))$ and $r_t(A_t)$, we have $|Z_i| \le 2\|\ws\|_1$
and it follows by Azuma's inequality (Lemma~\ref{lem:azuma}) that with
probability at least $1-\delta/2$:
\begin{align*}
\sum_{t=1}^T|Z_t| \le 2\|\ws\|_1 \sqrt{2 T \ln(2/\delta)}.
\end{align*}
To control the mean, we use event $\Ecal$, which, by Theorem~\ref{thm:variance_deviation} and Lemma~\ref{lem:reward_deviation}, holds with probability at least $1-\delta/2$.
By another union bound, with probability at least $1-\delta$, the regret of the algorithm is bounded by:
\begin{align*}
\mbox{Regret} &\le 2\|\ws\|_1\sqrt{2T\ln(2/\delta)} + \sum_{t=1}^T\sum_{\pi \in \Pi} \tilde{Q}_{t-1}^{\mu_{t-1}}(\pi)\Reg(\pi)\\
& \le 2\|\ws\|_1\sqrt{2T\ln(2/\delta)} + \sum_{t=1}^T\sum_{\pi \in \Pi}\BigBracks{(1-K\mu_{t-1})\tilde{Q}_{t-1}(\pi)\Reg(\pi) + \|\ws\|_1K\mu_{t-1}}\\
& \le 2\|\ws\|_1 \sqrt{2T\ln(2/\delta)} + \sum_{t=1}^T(4\psi + c_0+1) \normratio LK \mu_{t-1}\\
& \le \normratio L \left[2\sqrt{2T \ln(2/\delta)} + 2(4\psi + c_0 + 1)\sqrt{\frac{KTd_T}{\pmin}}\right]\\
\end{align*}
Here the first inequality is from the application of Azuma's
inequality above.  The second one uses the definition of
$\tilde{Q}^{\mu_{t-1}}_{t-1}$ to split into rounds where we play as
$\tilde{Q}_{t-1}$ and rounds where we explore.  The exploration rounds
occur with probability $K\mu_{t-1}$, and on those rounds we suffer
regret at most $\|\ws\|_1$.  For the other rounds, we use
Lemma~\ref{lem:reg_constraint_use} and then Lemma~\ref{lem:mu_sum}.
We collect terms using the inequality $\norm{\ws}_1\le L\|\ws\|_2^2/\|\ws\|_1$.
\end{proof}

\section{Proof of Oracle Complexity Bound in Theorem~\ref{thm:semimonster}}
\label{sec:optimization_full}

In this section we prove the oracle complexity bound in Theorem~\ref{thm:semimonster}.
First we describe how the optimization problem \OP can be solved via a coordinate ascent procedure.
Similar to the previous appendix, we use the shorthand $Q(a\given x)$ to mean $Q(a\in A\given x)$ for any projected
subdistribution $Q(A\given x)$. If $Q$ is a distribution, we have $\sum_{a\in\Acal}Q(a\given x)=L$. For
a subdistribution, this number can be smaller. 

\begin{algorithm}[t]
\caption{Coordinate Ascent Algorithm for Semi-Bandit Optimization Problem (\OP)}
\label{alg:coordinate_ascent}
\begin{algorithmic}[1]
\REQUIRE History $H$ and smoothing parameter $\mu$.
\STATE Initialize weights $Q\gets 0 \in \simplex(\Pi)$.
\WHILE{\texttt{true}}
\STATE For all $\pi$, define:
\begin{footnotesize}
\begin{align*}
V_\pi(Q) &\coloneqq \hat{\EE}_{x \sim H}\left[\sum_{\ell=1}^{L}\frac{1}{Q^\mu(\pi(x)_\ell|x)}\right],
&
S_\pi(Q) &\coloneqq \hat{\EE}_{x \sim H}\left[\sum_{\ell=1}^L\frac{1}{Q^\mu(\pi(x)_\ell|x)^2}\right],\\
D_\pi(Q) &\coloneqq V_\pi(Q) - \frac{2KL}{\pmin} - b_\pi
\end{align*}
\end{footnotesize}
\STATE If $\sum_{\pi} Q(\pi)(\frac{2KL}{\pmin} + b_\pi) >\frac{2KL}{\pmin}$, replace $Q$ by $cQ$ where $c \coloneqq \frac{2KL/\pmin}{\sum_\pi Q(\pi)(2KL/\pmin + b_\pi)} < 1$.
\STATE Else if $\exists \pi$ s.t. $D_\pi(Q) > 0$, update $Q(\pi)\gets Q(\pi) + \alpha_\pi(Q)$ where $\alpha_\pi(Q) \coloneqq \frac{V_\pi(Q) + D_\pi(Q)}{2(1-K\mu)S_\pi(Q)}$.
\STATE Otherwise halt and output $Q$. 
\ENDWHILE
\end{algorithmic}
\end{algorithm}

This problem is similar to the one used by \citet{agarwal2014taming} for contextual bandits rather than semibandits, and following their approach,
we provide a coordinate ascent procedure in the policy space (see Algorithm~\ref{alg:coordinate_ascent}).
There are two types of updates in the algorithm.
If the weights $Q$ are too large or the regret constraint in Equation~\ref{eq:op_regret} is violated, the algorithm multiplicatively shrinks all of the weights.
Otherwise, if there is a policy that is found to violate the variance constraint in Equation~\ref{eq:op_variance}, the algorithm adds weight to that policy, so that the constraint is no longer violated.

First, if the algorithm halts, then both of the conditions must be
satisfied.  The regret condition must be satisfied since we know that
$\sum_\pi Q(\pi) (2KL/\pmin + b_\pi) \le 2KL/\pmin$ which in
particular implies that $\sum_\pi Q(\pi)b_\pi \le 2KL/\pmin$ as
required.  Note that this also ensures that $\sum_\pi Q(\pi) \le 1$ so
$Q \in \simplex(\Pi)$.  Finally, if we halted, then for each $\pi$,
we must have $D_\pi(Q) \le 0$ which implies $V_\pi(Q) \le
\frac{2KL}{\pmin} + b_\pi$ so the variance constraint is also
satisfied.


The algorithm can be implemented by first accessing the oracle on the
importance weighted history $\{(x_{\tau}, \hat{y}_{\tau},
\ws)\}_{\tau=1}^{t}$ at the end of round $t$ to obtain $\pi_{t}$,
which we also use to compute $b_\pi$.  The low regret check in Step 4
of Algorithm~\ref{alg:coordinate_ascent} can be done efficiently,
since each policy in the support of the current distribution $Q$ was
added at a previous iteration of
Algorithm~\ref{alg:coordinate_ascent}, and we can store the regret of
the policy at that time for no extra computational burden. This allows
us to always maintain the expected regret of the current distribution
$Q$ for no added cost. Finding a policy violating the variance check
can be done by one call to the \amo. At round $t$, we create a dataset of the form $(x_i, z_i,
v_i)$ of size $2t$. The first $t$ terms come from the variance
$V_\pi(Q)$ and the second $t$ terms come from the rescaled empirical
regret $b_\pi$. For $\tau \leq t$, we define $x_\tau$ to be the
$\tau^{\textrm{th}}$ context,
\[
z_\tau(a) \coloneqq \frac{1}{t Q^\mu(a|x_\tau)}, \quad \mbox{and} \quad v_\tau
\coloneqq \mathbf{1}.
\]
With this definition, it is easily seen that $V_\pi(Q) = \sum_{\tau = 1}^t v_\tau^Tz_\tau(\pi(x_\tau))$.
For $\tau > t$, we define $x_\tau$ to be the context from round $\tau
- t$ and
\[
z_{\tau}(a) \coloneqq \frac{-\|\ws\|_1}{\|\ws\|_2^2 t\psi \mu \pmin}
\hat{y}_\tau(a), \quad \mbox{and} \quad v_\tau \coloneqq \ws.
\]
%
It can now be verified that $\sum_{\tau = t+1}^{2t} v_\tau^Tz_\tau $
recovers the $b_\pi$ term up to additive constants independent of the
policy $\pi$ (essentially up to the $\eta_t(\pi_t)$ term). Combining
everything, it can be checked that:
\begin{align*}
D_\pi(Q) = \sum_{\tau=1}^{2t} \langle z_\tau(\pi(x_\tau)),  v_\tau\rangle - \frac{2KL}{\pmin} - \invnormratio \frac{\eta_t(\pi_t)}{\psi \mu \pmin}
\end{align*}
The two terms at the end are independent of $\pi$ so by calling the
argmax oracle with this $2t$ sized dataset, we can find the policy
$\pi$ with the largest value of $D_\pi$. If the largest value is
non-positive, then no constraint violation exists. If it is strictly
positive, then we have found a constraint violator that we use to update the
probability distribution.

As for the iteration complexity, we prove the following theorem.
\begin{theorem}
For any history $H$ and parameter $\mu$, Algorithm~\ref{alg:coordinate_ascent} halts and outputs a set of weights $Q \in \simplex(\Pi)$ that is feasible for \OP.
Moreover, Algorithm~\ref{alg:coordinate_ascent} halts in no more than $\frac{8 \ln(1/(K\mu))}{\mu \pmin}$ iterations and each iteration can be implemented efficiently, with at most one call to $\amo$.
\label{thm:optimization}
\end{theorem}

Equipped with this theorem, it is easy to see that the total number of calls to the $\amo$ over the course of the execution of Algorithm~\ref{alg:semimonster} can be bounded as $\tilde{O}\left(T^{3/2} \sqrt{\frac{K}{\pmin\log(N/\delta)}}\right)$ by the setting of $\mu_t$.
Moreover, due to the nature of the coordinate ascent algorithm, the weight vector $Q$ remains sparse, so we can manipulate it efficiently and avoid running time that is linear in $N$.
As mentioned, this contrasts with the exponential-weights style algorithm of \citet{kale2010non} which maintains a dense weight vector over $\simplex(\Pi)$.

We mention in passing that \citet{agarwal2014taming} also develop two
improvements that lead to a more efficient algorithm.  They partition
the learning process into epochs and only solve \OP once every epoch,
rather than in every round as we do here (Lemma 2
in~\citet{agarwal2014taming}).  They also show how to use the weight
vector from the previous round to warm-start the next coordinate
ascent execution (Lemma 3 in~\citet{agarwal2014taming}).  Both of
these optimizations can also be implemented here, and we expect they
will reduce the total number of oracle calls over $T$ rounds to scale with
$\sqrt{T}$ rather than $T^{3/2}$ as in our result.  We omit these
details to simplify the presentation.



\subsection{Proof of Theorem~\ref{thm:optimization}}

Throughout the proof we write $U(A\given x)$ instead of $U_x(A)$ to parallel the notation $Q(A\given x)$. Also, similarly to $Q(a\given x)$, we
write $U(a\given x)$ to mean $U_x(a\in A)$.

We use the following potential function for the analysis, which is adapted from \citet{agarwal2014taming},
\begin{align*}
\Phi(Q) \coloneqq \frac{\hat{\EE}_{x \sim H}\Bracks{\bigRE{U(\cdot\given x)}{Q^\mu(\cdot\given x)}}}{1-K\mu} + \frac{\sum_{\pi} Q(\pi)b_\pi}{2K/\pmin}
\end{align*}
with
\begin{align*}
\RE{p}{q} \coloneqq \sum_{a \in \Acal}p_a \ln(p_a/q_a) + q_a - p_a
\end{align*}
being the unnormalized relative entropy. Its arguments $p$ and $q$ can be any non-negative vectors in~$\RR^K$.
For intuition, note that the partial derivative of the potential function with respect to a coordinate $Q(\pi)$ relates to the variance $V_\pi(Q)$ as follows:
\begin{align*}
\frac{\partial \Phi(Q)}{\partial Q(\pi)} &=
   \frac{ \hat{\EE}_{x \sim H}\Bracks{
               \sum_{a\in\pi(x)} \Parens{- \frac{U(a\given x)}{Q^\mu(a\given x)} (1-K\mu) + (1-K\mu) }
   }}{1-K\mu}
   + \frac{b_\pi}{2K/\pmin}
\\[3pt]
& =
   -\hat{\EE}_{x \sim H}\Bracks{
               \sum_{a\in\pi(x)} \frac{U(a\given x)}{Q^\mu(a\given x)}}
   +L+\frac{\pmin b_\pi}{2K}
\\[3pt]
& \le -\frac{\pmin}{K} V_\pi(Q) + L + \frac{\pmin b_\pi}{2K}
\\[3pt]
& = \frac{\pmin}{2K} \left( - 2V_\pi(Q) + \frac{2KL}{\pmin} + b_\pi\right)
\\[3pt]
& = \frac{\pmin}{2K} \bigParens{ -D_\pi(Q) - V_\pi(Q) }.
\end{align*}
This means that if $D_\pi(Q) > 0$, then the partial derivative is very negative, and by increasing the weight $Q(\pi)$, we can decrease the potential function $\Phi$.

We establish the following five facts:
\begin{enumerate}
\item $\Phi(0) \le L \ln(1/(K\mu))/(1-K\mu)$.
\item $\Phi(Q)$ is convex in $Q$.
\item $\Phi(Q) \ge 0$ for all $Q$.
\item The shrinking update, when the regret constraint is violated, does not increase the potential.
More formally, for any $c < 1$, we have $\Phi(cQ) \le \Phi(Q)$ whenever $\sum_\pi Q(\pi) (2KL/\pmin + b_\pi) > 2KL/\pmin$.
\item The additive update, when $D_\pi > 0$ for some $\pi$, lowers the potential by at least $\frac{L\mu\pmin}{4 (1-K\mu)}$.
\end{enumerate}
With these five facts, establishing the result is straightforward.  In
every iteration, we either terminate, perform the shrinking update, or
the additive update.  However, we will never perform the shrinking
update in two consecutive iterations, since our choice of $c$, ensures
the condition is not satisfied in the next iteration.  Thus, we
perform the additive update at least once every two iterations.
If we perform $I$ iterations, by the fifth fact, we are guaranteed to decrease the potential $\Phi$ by,
\begin{align*}
  \frac{I}{2} \frac{L\mu\pmin}{4(1-K\mu)} = \frac{IL\mu\pmin}{8(1-K\mu)}
\end{align*}
However, the total change in potential is bounded by $L
\ln(1/(K\mu))/(1-K\mu)$ by the first and second facts. Thus, we must have
\begin{align*}
\frac{IL\mu\pmin}{8(1-K\mu)} \le \frac{L\ln(1/(K\mu))}{(1-K\mu)},
\end{align*}
which is precisely the claim.

We now turn to proving the five facts.
The first three are fairly straightforward and the last two follow from analogous claims as in \citet{agarwal2014taming}.
To prove the first fact, note that the exploration distribution in $Q^\mu$ is exactly $U_x$, so
\begin{align*}
\Phi(0) = \hat{\EE}_{x \sim H}\Bracks{\sum_{a \in \Acal} \frac{U(a\given x) \ln\Parens{\frac{U(a\given x)}{K\mu U(a\given x)}} - (1-K\mu)U(a\given x)}{1-K\mu}} \le \frac{L \ln(1/(K\mu))}{1-K\mu},
\end{align*}
because $\sum_{a\in\Acal} U(a\given x)=L$ since $U(A\given x)$ is a distribution.
Convexity of this function follows from the fact that the unnormalized relative entropy is convex in the second argument, and the fact that the weight vector $q\in\RR^K$ with components
$q_a=Q^\mu(a\given x)$ is a linear transformation of $Q\in\RR^N$.
The third fact follows by the non-negativity of both the empirical regret $b_\pi$ and the unnormalized relative entropy $\RE{\cdot}{\cdot}$.
For the fourth fact, we prove the following lemma.
\begin{lemma}
Let $Q$ be a weight vector for which $\sum_\pi Q(\pi)(2KL/\pmin + b_\pi) > 2KL/\pmin$ and define $c\coloneqq\frac{2KL/\pmin}{\sum_\pi Q(\pi)(2KL/\pmin + b_\pi)} < 1$.
Then $\Phi(cQ) \le \Phi(Q)$.
\end{lemma}
\begin{proof}
Let $g(c) \coloneqq \Phi(cQ)$ and $Q_c^\mu(a | x) \coloneqq (1-K\mu)cQ(a|x) + K \mu U(a | x)$.
By the chain rule, using the calculation of the derivative above, we have:
\begin{align}
\notag
g'(c) &= \sum_{\pi} Q(\pi) \frac{\partial \Phi(cQ)}{\partial Q(\pi)}\\
\label{eq:potential:cQ:1}
& = \frac{\pmin }{2K}\sum_{\pi} Q(\pi)\left(\frac{2KL}{\pmin} + b_\pi\right) - \sum_{\pi} Q(\pi)\hat{\EE}_x\!\!\Bracks{\sum_{a \in \pi(x)} \frac{U(a|x)}{Q_c^\mu(a|x)}}.
\end{align}
Analyze the last term:
\begin{align}
\notag
\sum_{\pi}Q(\pi) \hat{\EE}_x\!\!\Bracks{\sum_{a \in \pi(x)} \frac{U(a|x)}{Q_c^\mu(a|x)}}
&= \hat{\EE}_x\!\!\Bracks{\sum_{a \in \Acal}\sum_{\pi \in \Pi} \frac{U(a|x) Q(\pi) \mathbf{1}(a \in \pi(x))}{Q_c^\mu(a | x)}}\\
\label{eq:potential:cQ:2}
&= \hat{\EE}_x\!\!\Bracks{\sum_{a \in \Acal} \frac{U(a|x) Q(a |x)}{Q_c^\mu(a|x)}}
 = \frac{1}{c}\hat{\EE}_x\!\!\Bracks{\sum_{a \in \Acal} \frac{U(a|x) cQ(a |x)}{Q_c^\mu(a|x)}}.
\end{align}
We now focus on one context $x$ and define $q_a \coloneqq cQ(a|x)$ and $u_a \coloneqq U(a|x)/L$. Note that $\sum_a U(a|x)=L$ so the vector $u$ describes a probability
distribution over $a\in\Acal$. The inner sum in \Eq{potential:cQ:2} can be upper bounded by:
\begin{align}
\notag
\sum_{a \in \Acal} \frac{U(a|x)cQ(a|x)}{Q_c^\mu(a|x)} & = \sum_{a \in \Acal} \frac{Lu_a q_a}{(1-K\mu)q_a + KL\mu u_a} = \sum_{a \in \Acal} \frac{Lu_a (q_a/u_a)}{(1-K\mu)(q_a/u_a) + KL\mu }\\
\notag
& = L\EE_{a\sim u}\Bracks{\frac{q_a/u_a}{(1-K\mu)(q_a/u_a) + KL\mu}}\\
\notag
& \le \frac{L\EE_{a\sim u}[q_a/u_a]}{(1-K\mu) \EE_{a\sim u} [q_a/u_a] + KL\mu}\\
\notag
& = \frac{L(\sum_{a \in \Acal}q_a)}{(1-K\mu)(\sum_{a \in \Acal}q_a) + KL\mu}\\
\label{eq:potential:cQ:3}
& \le \frac{L^2}{(1-K\mu)L + KL\mu} = L.
\end{align}
In the third line we use Jensen's inequality, noting that $x/(ax+b)$ is concave in $x$ for $a \ge 0$.
In \Eq{potential:cQ:3}, we use that $\sum_{a \in \Acal}q_a \le L$ and that $x/(ax+b)$ is non-decreasing, so plugging in $L$ for $\sum_a q_a$ gives an upper bound.

Combining Eqs.~\eqref{eq:potential:cQ:1}, \eqref{eq:potential:cQ:2}, and \eqref{eq:potential:cQ:3},
and plugging in our choice of $c = \frac{2KL/\pmin}{\sum_\pi Q(\pi)(2KL/\pmin + b_\pi)}$,
we obtain the following lower bound on $g'(c)$:
\begin{align*}
g'(c) &\ge \frac{\pmin }{2K}\sum_\pi Q(\pi)\left( \frac{2KL}{\pmin} + b_\pi\right) - \frac{L}{c}\\
& = \frac{\pmin }{2K} \left(\sum_\pi Q(\pi)\left(\frac{2KL}{\pmin} + b_\pi\right) - \frac{2KL}{c \pmin} \right) = 0.
\end{align*}

Since $g$ is convex, this means that $g(c')$ is nondecreasing for all values $c'$ exceeding $c$.
Since $c < 1$, we have:
\begin{align*}
\Phi(Q) = g(1) \ge g(c) = \Phi(cQ). \tag* \qedhere
\end{align*}
\end{proof}

And for the fifth fact, we have:
\begin{lemma}
Let $Q$ be a subdistribution and suppose, for some policy $\pi$, that $D_\pi(Q) > 0$.
Let $Q'$ be the new set of weights which is identical except that $Q'(\pi) \coloneqq Q(\pi) + \alpha$ with $\alpha \coloneqq \alpha_\pi(Q)>0$.
Then
\begin{align*}
\Phi(Q) - \Phi(Q') \ge \frac{L\mu\pmin}{4 (1-K\mu)}.
\end{align*}
\end{lemma}
\begin{proof}
Assume $D_\pi(Q)>0$.
Note that the updated subdistribution equals $Q'(\cdot) = Q(\cdot) + \alpha \mathbf{1}(\cdot = \pi)$, so its smoothed projection,
$Q'^\mu(a\given x) = Q^\mu(a\given x) + (1-K\mu)\alpha\mathbf{1}(a \in \pi(x))$, differs only in a small number of coordinates from $Q^\mu(a\given x)$. Using the shorthand
$q_a^\mu\coloneqq Q^\mu(a\given x)$, $q'^\mu_a\coloneqq Q'^\mu(a\given x)$ and $u_a\coloneqq U(a\given x)$, we have:
\begin{align*}
2K\bigParens{\Phi(Q) - \Phi(Q')} &= 2K\Parens{\frac{\hat{\EE}_x\BigBracks{\sum_{a}\BigParens{u_a \ln(u_a/q^\mu_a) - u_a \ln(u_a/q'^\mu_a) + q^\mu_a - q'^\mu_a}}}{(1-K\mu)} - \frac{\alpha b_\pi}{2K/\pmin} }\\
& = \frac{2K}{1-K\mu}\hat{\EE}_x\!\!\Bracks{\sum_{a \in \pi(x)} u_a \ln\left(\frac{q'^\mu_a}{q^\mu_a}\right)} -2K\alpha L - \alpha b_\pi \pmin\\
& \ge \frac{2 \pmin}{1-K\mu}\hat{\EE}_x\!\!\Bracks{\sum_{a \in \pi(x)} \ln\left(1 + \frac{\alpha(1-K\mu)}{Q^\mu(a\given x)}\right)} - \pmin\alpha\left(\frac{2KL}{\pmin} + b_\pi\right).
\end{align*}
The term inside the expectation can be bounded using the fact that $\ln(1+x) \ge x - x^2/2$ for $x\ge 0$:
\begin{align*}
\hat{\EE}_x\!\!\Bracks{\sum_{a \in \pi(x)} \ln\left(1 + \frac{\alpha(1-K\mu)}{Q^\mu(a\given x)}\right)}
&\ge \hat{\EE}_x\!\!\Bracks{\sum_{a \in \pi(x)}\Parens{\frac{\alpha(1-K\mu)}{Q^\mu(a\given x)} - \frac{\alpha^2(1-K\mu)^2}{2Q^\mu(a\given x)^2}}}\\
& = \alpha(1-K\mu) V_\pi(Q) - \frac{\alpha^2(1-K\mu)^2}{2}S_\pi(Q).
\end{align*}
Plugging this in the previous derivation gives a lower bound:
\begin{align*}
2K\bigParens{\Phi(Q) - \Phi(Q')} &\ge 2\pmin \alpha V_\pi(Q) - (1-K\mu)\pmin\alpha^2S_\pi(Q) - \pmin \alpha\Parens{\frac{2KL}{\pmin} + b_\pi}\\
& \ge \pmin \alpha\bigParens{V_\pi(Q) + D_\pi(Q)} - (1-K\mu)\pmin \alpha^2 S_\pi(Q),
\end{align*}
using the definition $D_\pi(Q) = V_\pi(Q) - \frac{2KL}{\pmin} - b_\pi$.
Since $\alpha = \frac{V_\pi(Q) + D_\pi(Q)}{2(1-K\mu) S_\pi(Q)}$, we obtain:
\begin{align*}
2K\bigParens{\Phi(Q) - \Phi(Q')} \ge \frac{\pmin \bigParens{V_\pi(Q) + D_\pi(Q)}^2}{4 (1-K\mu) S_\pi(Q)}
\end{align*}
Note that $S_\pi(Q) \ge \frac{1}{\mu\pmin}V_\pi(Q)$ (by bounding the square terms in the definition of $S_\pi(Q)$ by a linear term times the lower bound, which is $\mu \pmin$) and that $V_\pi(Q) > \frac{2KL}{\pmin}$ since $D_\pi(Q) > 0$.
Therefore:
\begin{align*}
2K\bigParens{\Phi(Q) - \Phi(Q')} &\ge \frac{\mu \pmin^2 \bigParens{V_\pi(Q) + D_\pi(Q)}^2}{4(1-K\mu)V_\pi(Q)}
 \ge \frac{\mu \pmin^2 V_\pi(Q)}{4 (1-K\mu)} \ge \frac{KL\mu \pmin}{2(1-K\mu)}.
\end{align*}
Dividing both sides of this inequality by $2K$ proves the lemma.
\end{proof}


\section{Proof of Theorem~\ref{thm:eels}}
\label{sec:eels}

The proof of Theorem~\ref{thm:eels} requires many delicate steps, so
we first sketch the overall proof architecture.  The first step is to
derive a parameter estimation bound for learning in linear models.
This is a somewhat standard argument from linear regression analysis,
and the important component is that the bound involves the 2nd-moment
matrix $\Sigma$ of the feature vectors used in the problem. Combining
this with importance weighting on the reward features $y$ as in
\semimonster, we prove that the policy used in the exploitation phase
has low expected regret, provided that $\Sigma$ has large eigenvalues.

The next step involves a precise characterization of the mean and
deviation of the 2nd-moment matrix $\Sigma$, which relies on
the exploration phase employing a uniform exploration strategy.  This
step involves a careful application of the matrix Bernstein inequality
(Lemma~\ref{lem:matrix_bernstein}).  We then bound the expected regret
accumulated during the exploration phase; we show, somewhat
surprisingly, that the expected regret can be related to the mean of
2nd-moment matrix $\Sigma$ of the reward features. Finally, since
per-round exploitation regret improves with a larger setting $\ls$,
while the cumulative exploration regret improves with a smaller
setting $\ls$, we optimize this parameter to balance the two terms.
Similarly, the per-round exploitation regret improves with a larger
setting $\ns$, while the cumulative exploration regret improves with a
smaller setting $\ns$, and our choice of $\ns$ optimizes this
tradeoff.

An important definition that will appear throughout the analysis is the expected reward variance, when a single action is chosen uniformly at random:
\begin{align}
V \coloneqq \EE_{(x,y) \sim \Dcal} \Bracks{\frac{1}{K}\sum_{a \in \Acal}y^2(a) - \biggParens{\frac{1}{K}\sum_{a \in \Acal} y(a)}^{\!2}}.
\label{eq:reward_variance}
\end{align}

\subsection{Estimating $V$}
The first step is a deviation bound for estimating $V$.
\begin{lemma}
\label{lem:v_estimation}
After $\ns$ rounds, the estimate $\hat{V}$ satisfies, with probability at least $1-\delta$,
\begin{align*}
|\hat{V} - V| \le \sqrt{\frac{V\ln(2/\delta)}{\ns}} + \frac{\ln(2/\delta)}{6\ns}.
\end{align*}
\end{lemma}
\begin{proof}
Note that our estimator, $\hat{V} = \frac{1}{\ns}\sum_{t=1}^{\ns} Z_t$, is an average of i.i.d.\ terms, with
\begin{align*}
  Z_t \coloneqq \frac{1}{2K^2}\sum_{a,b \in \Acal}(y_t(a) - y_t(b))^2\frac{\mathbf{1}(a,b \in A_t)}{U(a,b \in A_t)},
\end{align*}
where $U$ is a uniform distribution over all rankings.
The mean of this random variable is precisely $V$:
\begin{align*}
\EE_{(x,y)\sim\Dcal,A\sim U}[Z_t]
&= \frac{1}{2K^2}\EE_{x,y}\Bracks{\sum_{a,b \in \Acal}(y(a) - y(b))^2}\\
&= \EE_{x,y}\Bracks{\frac{1}{2K^2}\sum_{a,b \in \Acal} \BigParens{y(a)^2 - 2y(a)y(b) + y(b)^2}}\\
&= \EE_{x,y}\left[ \frac{1}{K} \sum_{a}y(a)^2 - \left(\frac{1}{K}\sum_a y(a)\right)^2\right] = V.
\end{align*}
Since we choose $L$ actions uniformly at
random, the probability for two distinct actions jointly being
selected is $U(a,b\in A) = \frac{L(L-1)}{K(K-1)}$ and for a single
action it is $U(a\in A) = L/K$.  The $(y(a) - y(b))^2$ term is at
most one but it is always zero for $a = b$, so the range of $Z_t$ is at most
\begin{align*}
0\le Z_t\le\frac{1}{2K^2} \sum_{a \ne b \in A_t} \frac{K(K-1)}{L(L-1)} = \frac{K(K-1)}{2K^2} \le \frac{1}{2}.
\end{align*}
Note that the last summation is only over the $L(L-1)$ action pairs
corresponding to the slate $A_t$, as the indicator in $Z_t$ eliminates
the other terms in the sum over all actions from $\Acal$.

As for the second moment, since $Z_t\in[0,1/2]$, we have
\begin{align*}
  \EE[Z_t^2] \leq \EE[Z_t]/2 \leq V/2.
\end{align*}
By Bernstein's inequality, we are guaranteed that with probability at least
$1-\delta$, after $\ns$ rounds,
\[
|\hat{V} - V| \le \sqrt{\frac{V\ln(2/\delta)}{\ns}} + \frac{\ln(2/\delta)}{6\ns}. \qedhere
\]
\end{proof}

Equipped with the deviation bound we can complete the square to find that
\begin{align*}
&V - \sqrt{\frac{V\ln(2/\delta)}{\ns}} + \frac{\ln(2/\delta)}{4\ns} \le \hat{V} + \frac{5\ln(2/\delta)}{12\ns}\\
&{\Rightarrow}\; \Parens{\sqrt{V} - \sqrt{\frac{\ln(2/\delta)}{4\ns}}}^2 \le \hat{V} + \frac{\ln(2/\delta)}{2\ns}\\
&{\Rightarrow}\; V \le \left(\sqrt{\frac{\ln(2/\delta)}{4\ns}} + \sqrt{\hat{V} + \frac{\ln(2/\delta)}{2\ns}}\right)^2
\le 2\hat{V} + \frac{3\ln(2/\delta)}{2\ns}.
\end{align*}
Our definition of $\ls$ uses $\tilde{V}$ which is precisely this final upper bound.
Working from the other side of the deviation bound, we know that
\begin{align*}
\hat{V} &\le \left(\sqrt{V} + \sqrt{\frac{\ln(2/\delta)}{4\ns}}\right)^2 \le 2V + \frac{\ln(2/\delta)}{2\ns}.
\end{align*}
And combining the two, we see that
\begin{align}
&V \le \tilde{V} \le 4V + \frac{5\ln(2/\delta)}{2\ns} \label{eq:vtilde_bound},
\end{align}
with probability at least $1-\delta$.

\subsection{Parameter Estimation in Linear Regression}
To control the regret associated with the exploitation rounds, we also need to bound $\|\hat{w} - \ws\|_2$ which follows from a standard analysis of linear regression.

At each round $t$, we solve a least squares problem with features $y_t(A_t)$ and response $r_t$ which we know has $\EE[r_t\given y_t,A_t]= y_t(A_t)^T\ws$.
The estimator is
\begin{align*}
w_t \coloneqq \argmin_{w} \sum_{i=1}^t \Parens{y_i(A_i)^Tw - r_i}^2.
\end{align*}
Define the 2nd-moment matrix of reward features,
\begin{align*}
\Sigma_t \coloneqq \sum_{i=1}^t y_i(A_i)y_i(A_i)^T,
\end{align*}
which governs the estimation error of the least squares solution as we show in the next lemma.
\begin{lemma}
\label{lem:eels_weight_estimation}
Let $\Sigma_t$ denote the 2nd-moment reward matrix after $t$ rounds of interaction and let $w_t$ be the least-squares solution.
There is a universal constant $c > 0$ such that for any $\delta \in (0,2/e)$, with probability at least $1-\delta$,
\begin{align*}
\|w_t - \ws\|_{\Sigma_t}^2 \le cL\ln(2/\delta).
\end{align*}
\end{lemma}
\begin{proof}
This lemma is the standard analysis of fixed-design linear
regression with bounded noise. By definition of the ordinary least
squares estimator, we have $\Sigma_t w_t = Y_{1:t}^Tr_{1:t}$ where
$Y_{1:t} \in \RR^{t \times L}$ is the matrix of features, $r_{1:t} \in
\RR^{t}$ is the vector of responses and $\Sigma_t=Y_{1:t}^TY_{1:t}$ is the 2nd-moment matrix
of reward features defined above.  The true weight vector satisfies $\Sigma_t\ws =
Y_{1:t}^T(r_{1:t} - \xi_{1:t})$ where $\xi_{1:t} \in
\RR^t$ is the noise.  Thus $\Sigma_t(w_t-\ws)=Y_{1:t}^T\xi_{1:t}$, and therefore,
\begin{align*}
\|w_t - \ws\|_{\Sigma_t}^2
= (w_t-\ws)^T\Sigma_t(w_t-\ws)
= (w_t-\ws)^T\Sigma_t\Sigma_t^\dagger\Sigma_t(w_t-\ws)
= \xi_{1:t}^T Y_{1:t}\Sigma_t^\dagger Y_{1:t}^T\xi_{1:t},
\end{align*}
where $\Sigma_t^\dagger$ is the pseudoinverse of $\Sigma_t$ and we use the fact
that $AA^\dagger A=A$ for any symmetric matrix~$A$.
Since $\Sigma_t^\dagger = (Y_{1:t}^TY_{1:t})^\dagger$, the matrix $Y_{1:t}\Sigma_t^\dagger Y_{1:t}^T$ is a projection
matrix, and it can be written as $UU^T$ where
$U \in \RR^{t\times L'}$ is a matrix with $L'$ orthonormal columns where $L'\le L$. We now
have to bound the term $\|U^T\xi_{1:t}\|_2^2 = \xi_{1:t}^TUU^T\xi_{1:t}$.  Let $H_{xy}=(x_1,y_1,\dotsc,x_t,y_t)$ denote the history
excluding the noise. Conditioned on $H_{xy}$, the vector
$\xi_{1:t}$ is a subgaussian random vector with independent
components, so we can apply subgaussian tail bounds.
Applying Lemma~\ref{lem:vector_concentration}, due to Rudelson and Vershynin~\citeapp{rudelson2013hanson}, we see that with probability at least $1-\delta$,
\begin{align}
\label{eq:noise:bound}
\xi_{1:t}^TUU^T \xi_{1:t} &\le  \EE\bigBracks{\xi_{1:t}^TUU^T \xi_{1:t}\bigGiven H_{xy}} + \sqrt{c_0\|UU^T\|_F^2 \ln(2/\delta)} + c_0\|UU^T\|\ln(2/\delta)\\
\notag
& \le \left(L + \sqrt{c_0L \ln(2/\delta)} + c_0\ln(2/\delta)\right)
 \le \left(\sqrt{L} + \sqrt{c_0\ln(2/\delta)}\right)^2.\notag
\end{align}
To derive the second line, we use the fact that $UU^T$ is a projection matrix for an $L'$-dimensional subspace, so its Frobenius norm
is bounded as $\|UU^T\|_F^2 = \tr(UU^T) = L'\le L$, while its spectral norm is $\|UU^T\| = 1$.
The expectation in \Eq{noise:bound} is bounded using the conditional independence of the noise and the fact that its conditional expectation is zero:
\begin{align*}
\EE\bigBracks{\xi_{1:t}^TUU^T \xi_{1:t}\bigGiven H_{xy}}
&= \tr\BigParens{UU^T \EE[\xi_{1:t}\xi_{1:t}^T\given H_{xy}]}
= \tr\BigParens{UU^T \diag(\EE[\xi_i^2\given x_i,y_i])_{i\le t}}
\\
&\le \tr(UU^T)\BigParens{\max_{i\le t} \EE[\xi_i^2\given x_i,y_i]} \le L' \le L.
\end{align*}
Finally, when $\delta \in (0,2/e)$ and with $c = (1 + \sqrt{c_0})^2$, we obtain the desired bound.
\end{proof}

\subsection{Analysis of the 2nd-Moment Matrix $\Sigma_t$}
We now show that the 2nd-moment matrix of reward features has large eigenvalues.
This lets us translate the error in Lemma~\ref{lem:eels_weight_estimation} to the Euclidean norm, which will play a role in bounding the exploitation regret.
Interestingly, the lower bound on the eigenvalues is related to the exploration regret, so we can explore until the eigenvalues are large, without incurring too much regret.

To prove the bound, we use a full sequence of exploration data, which enables us to bypass the data-dependent stopping time.
Let $\{x_t,y_t,A_t,\xi_t\}_{t=1}^T$ be a sequence of random variables where $(x,y,\xi) \sim \Dcal$ and $A_t$ is drawn uniformly at random.
Let $w_t$ be the least squares solution on the data in this sequence up to round $t$, and let $\Sigma_t$ be the 2nd-moment matrix of the reward features.

\begin{lemma}
\label{lem:eels_covariance_conditioning}
With probability at least $1-\delta$, for all $t \le T$,
\begin{align*}
\Sigma_t \succeq \left(tV - 4L\sqrt{tV \ln(4LT/\delta)} -
4L\ln(4LT/\delta)\right)I_L,
\end{align*}
where $I_L$ is the $L\times L$ identity matrix.
\end{lemma}
\begin{proof}
For $K=1$, we have $V=0$, so the bound holds. In the remainder, assume $K\ge 2$.
The proof has two components: the spectral decomposition of the mean $\EE\Sigma_t$ and the deviation bound on $\Sigma_t$.

\textbf{Spectral decomposition of $\EE\Sigma_t$:} The first step in the proof is to
analyze the expected value of the 2nd-moment matrix.  Since $y_t, A_t$ are
identically distributed, it suffices to consider just one term. Fixing
$x$ and $y$, we only reason about the randomness in picking $A$. Let
$S \coloneqq \EE_{A \sim U}[y(A)y(A)^T] \in \RR^{L\times L}$ be
the mean matrix for that round.  We have:
\begin{align*}
z^TS z & = \sum_{\ell=1}^L z_\ell^2 \sum_{a \in \Acal} \frac{1}{K} y(a)^2 + \sum_{\ell \ne \ell'} z_\ell z_{\ell'} \sum_{a \ne a' \in \Acal} \frac{y(a)y(a')}{K(K-1)}\\
& = \frac{\|y\|_2^2 \|z\|_2^2}{K} + \sum_{\ell\ne \ell'}z_\ell z_\ell' \sum_{a,a' \in \Acal} \frac{y(a)y(a')}{K(K-1)} - \sum_{\ell\ne \ell'} z_\ell z_\ell' \sum_{a} \frac{y(a)^2}{K(K-1)}.
\intertext{%
Define $\bar{y} \coloneqq \frac{1}{K}\sum_{a \in \Acal} y(a)$, $E_y^2 \coloneqq \frac{1}{K}\sum_{a \in \Acal} y(a)^2$, and $V_y \coloneqq E_y^2 - \bar{y}^2$, and observe that
by the definition of $V$ in \Eq{reward_variance}, we have $\EE_{x,y} V_y = V$.
Continuing the derivation, we obtain:}
z^TS z & = E_y^2 \|z\|_2^2 + \sum_{\ell \ne \ell'} z_\ell z_{\ell'}\left( \frac{K}{K-1}\bar{y}^2 - \frac{1}{K-1}E_y^2\right)\\
& = E_y^2 \|z\|_2^2 + \BigParens{(z^T\vone)^2 - \norm{z}_2^2}\left( \frac{K}{K-1}\bar{y}^2 - \frac{1}{K-1}E_y^2\right)\\
& = \frac{K}{K-1}V_y\|z\|_2^2 + (z^T\vone)^2\left( \frac{K}{K-1}\bar{y}^2 - \frac{1}{K-1}E_y^2\right).
\intertext{%
To finish the derivation, let $u=\vone/\sqrt{L}$ be the unit vector in the direction of all ones and $P=I-uu^T$ be the projection matrix
on the subspace orthogonal with $u$. Then}
z^TS z
&=
\frac{K}{K-1} V_y(z^Tuu^Tz + z^TPz) + L(z^Tu u^Tz)\left( \bar{y}^2 - \frac{1}{K-1}V_y\right)
\\
&=s
\Parens{\frac{K-L}{K-1}V_y + L\bar{y}^2}(z^T uu^T z) + \frac{K}{K-1}V_y(z^T P z).
\end{align*}
Thus,
\[
  S =  \Parens{\frac{K-L}{K-1}V_y + L\bar{y}^2} uu^T + \frac{K}{K-1}V_y P.
\]
By taking the expectation, we obtain the spectral decomposition with eigenvalues $\lambda_u$ and $\lambda_P$ associated, respectively, with $uu^T$ and $P$:
\begin{equation}
\label{eq:spectral}
  \EE_{x,y,A}[y(A)y(A)^T]=\EE_{x,y}[S]= \underbrace{\Parens{\frac{K-L}{K-1}V + L\EE[\bar{y}^2]}}_{\lambda_u} uu^T + \underbrace{\Parens{\frac{K}{K-1}V}}_{\lambda_P}P.
\end{equation}
We next bound the eigenvalue $\lambda_u$.
By positivity of $y$, note that $E_y^2\le\bigParens{\max_a y(a)}\bar{y}\le K\bar{y}^2$. Therefore,
$V_y=E_y^2-\bar{y}^2\le(K-1)\bar{y}^2$, and thus $\EE[\bar{y}^2]\ge V/(K-1)$, so
\[
  \lambda_u = \frac{K-L}{K-1}V + L\EE[\bar{y}^2] \ge \frac{K}{K-1}V
.
\]
Thus, both eigenvalues are lower bounded by $\frac{K}{K-1}V\ge V$.

\textbf{The deviation bound:}
For deviation bound, we follow the spectral structure of $\EE\Sigma_t$ and first reason
about the properties of $\Sigma_t u$, followed by the analysis of $P\Sigma_t P$.
Throughout the analysis, let $z_i \coloneqq y_i(A_i)$ denote the $L$-dimensional reward feature
vector on round $i$, and consider a fixed $t\le T$.

\textbf{Direction $u$:} We begin by the analysis of $\Sigma_t u$. Specifically, we will show
that $\norm{\Sigma_t u - (\EE\Sigma_t) u}_2$ is small. We apply Bernstein's inequality
to a single coordinate $\ell$, then take a union bound to obtain a bound on $\norm{\cdot}_\infty$, and convert to a bound on $\norm{\cdot}_2$. For a fixed $\ell$ and $i\le t$, define
\[
  X_i\coloneqq z_{i\ell} z_i^T u
\]
and note that $(\Sigma_t u)_\ell = \sum_{i\le t} X_i$. The range and variance of $X_i$ are bounded as
\begin{align*}
  &0\le X_i  \le\sqrt{L}\\
  &\EE[X_i^2]=\EE[z_{i\ell}^2 (z_i^T u)^2]\le\EE[(z_i^T u)^2]=u^T\EE[z_iz_i^T]u=\lambda_u
\end{align*}
where the last equality follows by \Eq{spectral}. Thus, by Bernstein's inequality, with probability at least $1-\delta/2L$,
\[
\Abs{
  \sum_{i\le t} X_i - \sum_{i\le t} \EE X_i
}
\le
\sqrt{2t\lambda_u\ln(4L/\delta)}+\sqrt{L}\ln(4L/\delta)/3
.
\]
Taking a union bound over $\ell\le L$ yields that with probability at least $1-\delta/2$,
\begin{equation}
\label{eq:Sigma:t:u}
\bigNorm{
  \Sigma_t u-(\EE\Sigma_t)u
}_2
\le
\sqrt{L}\bigNorm{
  \Sigma_t u-(\EE\Sigma_t)u
}_\infty
\le
\sqrt{2Lt\lambda_u\ln(4L/\delta)}+L\ln(4L/\delta)/3
.
\end{equation}

\textbf{Orthogonal to $u$:}
In the subspace orthogonal to $u$, we apply the matrix Bernstein inequality. Let $X_i$, for $i\le t$, be the matrix random variable
\[
  X_i\coloneqq Pz_iz_i^TP-P\EE[z_iz_i^T]P
\]
and note that $\sum_{i\le t} X_i=P\Sigma_t P-\EE[P\Sigma_t P]$. Since $z_i$ are i.i.d., below we analyze a single $z_i$ and $X_i$ and drop the index $i$.
The range can be bounded as
\[
  \lambda_{\max}(X)\le\lambda_{\max}(Pzz^TP)\le\norm{z}_2^2\le L.
\]
To bound the variance, we use Schatten norms, i.e., $L_p$ norms applied to the spectrum of a symmetric matrix. The Schatten $p$-norm is denoted
as $\norm{\cdot}_{\sigma,p}$. Note that the operator norm is $\norm{\cdot}_{\sigma,\infty}$ and the trace norm is $\norm{\cdot}_{\sigma,1}$. We begin by upper-bounding the variance by the second moment,
then use the convexity of the norm, the monotonicity of Schatten norms, and the fact that the trace norm of a positive semi-definite matrix
equals its trace to obtain:
\begin{align*}
  \BigNorm{
  \EE[X^2]}_{\sigma,\infty}
&\le
  \BigNorm{
  \EE[(Pzz^TP)^2]}_{\sigma,\infty}
\\
&\le
  \EE\BigBracks{\norm{(Pzz^TP)^2}_{\sigma,\infty}}
 \le
  \EE\BigBracks{\norm{(Pzz^TP)^2}_{\sigma,1}}
 =
  \EE\BigBracks{\tr\BigParens{(Pzz^TP)^2}},
\\
\intertext{%
and continue by the matrix Holder inequality, $\tr(A^TB)\le\norm{A}_{\sigma,\infty}\norm{B}_{\sigma,1}$,
and \Eq{spectral} to obtain:}
 &\le
  \BigParens{
  \max_z\norm{Pzz^TP}_{\sigma,\infty}}
  \EE\BigBracks{\norm{Pzz^TP}_{\sigma,1}}
\\
 &\le
  L\tr\EE\bigBracks{Pzz^TP}
 =
  L(L-1)\lambda_P.
\end{align*}
Reverting to the notation $\norm{\cdot}$ for the operator norm, the matrix Bernstein inequality (\Lem{matrix_bernstein}) yields that with probability at least $1-\delta/2$,
\begin{align}
\label{eq:pi_perp_deviation}
\BigNorm{P\Sigma_tP-\EE[P\Sigma_tP]}
&
=
\BigNorm{\sum_{i\le t} X_i- \sum_{i\le t}\EE X_i} \le \sqrt{2L^2t\lambda_P\ln(2L/\delta)} + 2L\ln(2L/\delta)/3.
\end{align}

\textbf{The final bound:} Let $x$ be an arbitrary unit vector. Decompose it along the all-ones direction and the orthogonal direction as $x=\alpha u +\beta v$, where $v\perp u$, and $\alpha^2+\beta^2=1$.
Let $v'=(\Sigma_t-\EE\Sigma_t) u$. Then
\begin{align}
\notag
 \Abs{x^T(\Sigma_t-\EE\Sigma_t)x}
 &=
 \BigAbs{\alpha u^T(\Sigma_t-\EE\Sigma_t)x
 +\beta v^T(\Sigma_t-\EE\Sigma_t)\alpha u
 +\beta v^T(\Sigma_t-\EE\Sigma_t)\beta v
 }
\\
\notag
 &\le
 \abs{\alpha}\cdot\bigNorm{u^T(\Sigma_t-\EE\Sigma_t)}_2
 +\abs{\alpha\beta}\cdot\bigNorm{(\Sigma_t-\EE\Sigma_t)u}_2
 +\beta^2\Abs{v^T(\Sigma_t-\EE\Sigma_t)v}
\\
\label{eq:bound:1}
 &\le 2\abs{\alpha}\cdot\norm{v'}_2
 +\abs{\beta}\cdot\Norm{P\Sigma_tP-\EE[P\Sigma_tP]}
\enspace.
\end{align}
From \Eq{spectral}, we have
\begin{equation}
\label{eq:bound:2}
 x^T(\EE\Sigma_t)x
 \ge
 \alpha^2 t\lambda_u+\beta^2 t\lambda_P.
\end{equation}
To finish the proof, we will use the identity valid for all $A,B,c\ge 0$
\begin{align}
\notag
 A+B-c\sqrt{A+B}
&\ge
 A+B-c\sqrt{A}-c\sqrt{B}+c^2/4-c^2/4
\\
\notag
&=
 A-c\sqrt{A}+\bigParens{\sqrt{B}-c/2}^2-c^2/4
\\
\label{eq:bound:identity}
&\ge
 A-c\sqrt{A}-c^2/4.
\end{align}
Combining \Eq{bound:1} and \Eq{bound:2}, and plugging in bounds from \Eq{Sigma:t:u} and \Eq{pi_perp_deviation}, we have
\begin{align*}
  x^T\Sigma_t x
&\ge
 \alpha^2 t\lambda_u+\beta^2 t\lambda_P
 -2\abs{\alpha}\cdot\norm{v'}_2
 -\abs{\beta}\cdot\Norm{P\Sigma_tP-\EE[P\Sigma_tP]}
\\
&\ge
 \alpha^2 t\lambda_u+\beta^2 t\lambda_P
 -2\abs{\alpha}
   \sqrt{2Lt\lambda_u\ln(4L/\delta)}
 -\frac{2\abs{\alpha}L}{3}\ln(4L/\delta)
\\
&\hphantom{{}\ge
 \alpha^2 t\lambda_u+\beta^2 t\lambda_P}
 -\abs{\beta}\sqrt{2L^2t\lambda_P\ln(2L/\delta)}
 -\frac{2\abs{\beta}L}{3}\ln(2L/\delta)
\\
&\ge
 \alpha^2 t\lambda_u+\beta^2 tV
 -\Parens{2\sqrt{2L\ln(4L/\delta)}}\sqrt{\alpha^2t\lambda_u}
\\
&\hphantom{{}\ge
 \alpha^2 t\lambda_u+\beta^2 tV}{}
 -\abs{\beta}\sqrt{4L^2tV\ln(4L/\delta)}
 -2L\ln(4L/\delta),
\intertext{%
where we used $V\le\lambda_P\le 2V$, and $\abs{\alpha}\le 1,\abs{\beta}\le 1$.
We now apply \Eq{bound:identity} with $A+B=\alpha^2t\lambda_u$ and $A=\alpha^2tV$ to obtain}
  x^T\Sigma_t x
&\ge
 \alpha^2 tV+\beta^2 tV
 -\Parens{2\sqrt{2L\ln(4L/\delta)}}\sqrt{\alpha^2tV}
 -2L\ln(4L/\delta)
\\
&\hphantom{{}\ge
 \alpha^2 tV+\beta^2 tV}
 -2L\abs{\beta}\sqrt{tV\ln(4L/\delta)}
 -2L\ln(4L/\delta)
\\
&\ge
 tV
 -2L\sqrt{2}\Parens{\abs{\alpha}+\abs{\beta}}\sqrt{tV\ln(4L/\delta)}
 -4L\ln(4L/\delta)
\\
&\ge
 tV
 -4L\sqrt{tV\ln(4L/\delta)}
 -4L\ln(4L/\delta),
\end{align*}
where we used $\abs{\alpha}+\abs{\beta}\le\sqrt{2\alpha^2+2\beta^2}=\sqrt{2}$.
The lemma follows by the union bound over $t\le T$.
\end{proof}

\subsection{Analysis of the Exploration Regret}

The analysis here is made complicated by the fact that the stopping time of the exploration phase is a random variable.
If we let $\hat{t}$ denote the last round of the exploration phase, this quantity is a random variable that depends on the history of interaction up to and including round $\hat{t}$.
Our proof here will use a non-random bound $\ts$ that satisfies $\PP(\hat{t} \le \ts) \ge 1-\delta$.
We will compute $\ts$ based on our analysis of the 2nd-moment matrix $\Sigma_t$. 

A trivial bound on the exploration regret is
\begin{align}
\sum_{t=1}^{\ts}\BigBracks{r_t(\pi^\star(x_t))- r_t(A_t)} \le \ts \|\ws\|_2 \sqrt{L},
\label{eq:exploration_regret_one}
\end{align}
which follows from the Cauchy-Schwarz inequality and the fact that the reward features are in $[0,1]$.

In addition, we also bound the exploration regret by the following more precise bound:
\begin{lemma}[Exploration Regret Lemma]
\label{lem:eels_exploration_regret}
Let $\ts$ be a non-random upper bound on the random variable $\hat{t}$ satisfying $\PP(\hat{t} \le \ts) \ge 1 -\delta$.
Then with probability at least $1-2\delta$, the exploration regret is
\begin{align*}
\sum_{t=1}^{\hat{t}}\BigBracks{r_t(\pi^\star(x_t)) - r_t(A_t)} \le \ts \|\ws\|_2\min\Set{\sqrt{KV},\sqrt{L}} + \|\ws\|_2\sqrt{2L\ts \ln(1/\delta)}.
\end{align*}
\end{lemma}
\begin{proof}
Let $\{x_t,y_t,A_t,\xi_t\}_{t=1}^T$ be a sequence of random variables where $(x,y,\xi) \sim \Dcal$ and $A_t$ is drawn uniformly at random.
We are interested in bounding the probability of the event
\begin{align*}
\Ecal \coloneqq \Braces{ \sum_{t=1}^{\hat{t}}\BigParens{y_t(\pi^\star(x_t)) - y_t(A_t)}^T\!\ws \le \epsilon }.
\end{align*}
This term is exactly the exploration regret, so we want to make sure the probability of this event is large.
We first apply the upper bound
\begin{align*}
\sum_{t=1}^{\hat{t}}\BigParens{y_t(\pi^\star(x_t)) - y_t(A_t)}^T\!\ws &\le \sum_{t=1}^{\hat{t}} \BigParens{y_t(A_t^\star) - y_t(A_t)}^T\!\ws,
\end{align*}
where $A_t^\star = \arg\max_A y_t(A)^T\ws$ is the best possible
ranking.  This upper bound ensures that every term in the sum
is non-negative.
We next remove the dependence on the random stopping time $\hat{t}$ and replace it with a deterministic number of terms $\ts$:
\begin{align*}
\PP(\Ecal) &\ge \PP\left(\sum_{t=1}^{\hat{t}} \BigParens{y_t(A_t^\star) - y_t(A_t)}^T\!\ws \le \epsilon\right)\\
& \ge \PP\left(\sum_{t=1}^{\hat{t}} \BigParens{y_t(A_t^\star) - y_t(A_t)}^T\!\ws \le \epsilon \;\cap\; \hat{t} \le \ts\right)\\
& \ge \PP\left(\sum_{t=1}^{\ts} \BigParens{y_t(A_t^\star) - y_t(A_t)}^T\!\ws \le \epsilon \;\cap\; \hat{t} \le \ts\right)\\
& \ge 1 - \PP\left(\sum_{t=1}^{\ts} \BigParens{y_t(A_t^\star) - y_t(A_t)}^T\!\ws > \epsilon\right) - \PP\left(\hat{t} > \ts\right)\\
& \ge 1 - \delta - \PP\left(\sum_{t=1}^{\ts} \BigParens{y_t(A_t^\star) - y_t(A_t)}^T\!\ws > \epsilon\right).
\end{align*}
The first line follows from the definition of $A_t^\star$ which only increases the sum, so decreases the probability of the event.
The second inequality is immediate, while the third inequality holds because all terms of the sequence are non-negative.
The fourth inequality is the union bound and the last is by assumption on the event $\{\hat{t}\le \ts\}$.

Now we can apply a standard concentration analysis.
The mean of the random variables is
\begin{align*}
\EE_{x,y,A}\BigBracks{\bigParens{y(A^\star) - y(A)}^T\!\ws} &\le \|\ws\|_2  \BigNorm{\EE_{x,y,A} \bigBracks{y(A^\star) - y(A)}}_2 \\
& = \|\ws\|_2  \sqrt{\sum_{\ell\le L} \EE_{x,y}\Bracks{y(A_\ell^\star) - \bar{y}}^2}\\
& \le \|\ws\|_2  \sqrt{\sum_{\ell\le L} \EE_{x,y}\Bracks{\bigParens{y(A_\ell^\star) - \bar{y}}^2}}\\
& \le \|\ws\|_2  \sqrt{K} \sqrt{\frac{1}{K}\sum_{a\in\Acal} \EE_{x,y}\Bracks{\bigParens{y(a) - \bar{y}}^2}}\\
& = \|\ws\|_2  \sqrt{KV}.
\end{align*}
The first inequality is Cauchy-Schwarz while the second is Jensen's inequality and the third comes from adding non-negative terms.
The range of the random variable is bounded as
\begin{align*}
\sup_{x,y,A} \biggAbs{\bigParens{y(A^\star) - y(A)}^T\!\ws  - \EE_{x,y,A}\BigBracks{\bigParens{y(A^\star) - y(A)}^T\!\ws}} \le \|\ws\|_2  \sqrt{L},
\end{align*}
because $0\le \bigParens{y(A^\star) - y(A)}^T\!\ws\le\norm{\ws}_2 \sqrt{L}$.
Thus by Hoeffding's inequality, with probability at least $1-\delta$,
\begin{align*}
\sum_{t=1}^{\ts} \bigParens{y(A^\star) - y(A)}^T\!\ws &\le \sum_{t=1}^{\ts}\EE_{x,y,A}\BigBracks{\bigParens{y(A^\star) - y(A)}^T\!\ws} + \|\ws\|_2 \sqrt{2 L \ts \ln(1/\delta)}\\
& \le \ts \|\ws\|_2 \sqrt{K V} + \|\ws\|_2 \sqrt{2L\ts\ln(1/\delta)}.
\end{align*}
Combining this bound with the bound of \Eq{exploration_regret_one} proves the lemma.
\end{proof}

\subsection{Analysis of the Exploitation Regret}
In this section we show that after the exploration rounds, we can find a policy that has low expected regret.
The technical bulk of this section involves a series of deviation bounds showing that we have good estimates of the expected reward for each policy.


In addition to $\bar{y}$ from the previous sections, we will also need the sample quantity $\bar{y}_t\coloneqq\frac1K\sum_{a\in\Acal}y_t(a)$,
which will allow us to relate the exploitation regret to the variance term $V$.
Since we are using uniform exploration, the importance-weighted feature vectors are as follows:
\begin{align*}
\hat{y}_t(a) = \frac{\mathbf{1}(a \in A_t)y_t(a)}{U(a \in A_t)} = \frac{K}{L}\mathbf{1}(a \in A_t)y_t(a).
\end{align*}
Given any estimate $\hat{w}$ of the true weight vector $\ws$, the empirical reward estimate for a policy $\pi$ is
\begin{align*}
\eta_n(\pi, \hat{w}) \coloneqq \frac{1}{n}\sum_{t=1}^n \hat{y}_t(\pi(x_t))^T\hat{w}.
\end{align*}
%
A natural way to show that the policy with a low empirical reward has also a low expected regret is to show that for all policies $\pi$, the empirical reward estimate $\eta_n(\pi,\hat{w})$ is close to the true reward, $\eta(\pi)$, defined as,
\begin{align*}
\eta(\pi) \coloneqq \EE_{x,y}\Bracks{y(\pi(x))^T\ws}.
\end{align*}
%
Rather than bounding the deviation of $\eta_n$ directly, we instead control a shifted version of $\eta_n$, namely,
\begin{align*}
\psi_n(\pi,\hat{w}) \coloneqq \frac{1}{n}\sum_{t=1}^n \Bracks{\hat{y}_t(\pi(x_t))^T\hat{w}- \bar{y}_t\mathbf{1}^T\hat{w}},
\end{align*}
%
where $\mathbf{1}$ is the
$L$-dimensional all-ones vector. Note that $\bar{y}_t$
is based on the rewards of all actions, even those that were not
chosen at round $t$. This is not an issue, since $\bar{y}_t$ is only used in the analysis.

\begin{lemma}
\label{lem:eels_reward_deviation}
Fix $\delta \in (0,1)$ and assume that $\|\hat{w} - \ws\|_2  \le \theta$ for some $\theta \ge 0$.
For any $\delta \in (0,1)$, with probability at least $1-\delta$, we have that for all $\pi \in \Pi$,
\begin{align*}
& \left|\psi_n(\pi, \hat{w}) - \eta(\pi,\ws) + \EE_{x,y}\!\Bracks{\bar{y}\mathbf{1}^T\ws}\right| \\
& \le 2(\theta + \|\ws\|_2)\sqrt{K}\left( \sqrt{\frac{\ln(2N/\delta)}{n}} + \sqrt{\frac{K}{L}}\frac{\ln(2N/\delta)}{n}\right) + \theta \min\{\sqrt{KV},2\sqrt{L}\}.
\end{align*}
\end{lemma}
\begin{proof}
We add and subtract several terms to obtain a decomposition.
We introduce the shorthands $y_\pi\coloneqq y(\pi(x))$, $\hat{y}_\pi\coloneqq\hat{y}(\pi(x))$, and $\hat{y}_{t,\pi}\coloneqq\hat{y}_t(\pi(x_t))$.
\begin{align*}
&\psi_n(\pi,\hat{w}) - \eta(\pi,\ws) + \EE_{x,y}\!\Bracks{\bar{y}\mathbf{1}^T\ws} \\
&\qquad{}
  = \frac{1}{n}\sum_{t=1}^n
  (\hat{y}_{t,\pi} - \bar{y}_t\mathbf{1})^T\hat{w} - \EE_{x,y}[y_{\pi} - \bar{y}\mathbf{1}]^T\ws\\
&\qquad{}
  =
  \underbrace{\frac{1}{n}\sum_{t=1}^n\BigParens{(\hat{y}_{t,\pi} - \bar{y}_t\mathbf{1})^T\hat{w} - \EE_{x,y}[y_\pi - \bar{y}\mathbf{1}]^T\hat{w}}}_{\textrm{Term 1}}
  + \underbrace{\vphantom{\sum_{t=1}^n}
    \EE_{x,y}[y_{\pi} - \bar{y}\mathbf{1}]^T(\hat{w}-\ws)}_{\textrm{Term 2}}
.
\end{align*}
There are two terms to bound here.
We bound the first term by Bernstein's inequality, using that fact that $\hat{y}_t$ is coordinate-wise unbiased for $y$.
The second term will be bounded via a deterministic analysis, which will yield an upper bound related to the reward-feature variance $V$.

\textbf{Term 1:} Note that each term of the sum has expectation zero, since $\hat{y}_t$ is an unbiased estimate.
Moreover, the range of each individual term in the sum can be bounded as
\begin{align*}
 \Abs{\bigParens{ \hat{y}_{t,\pi} - \bar{y}_t\mathbf{1} - \EE_{x,y}[y_\pi - \bar{y}\mathbf{1}]}^T\hat{w}
 }
& \le
 \|\hat{w}\|_2
 \BigNorm{ \hat{y}_{t,\pi} - \bar{y}_t\mathbf{1} - \EE_{x,y}[y_\pi - \bar{y}\mathbf{1}]}_2  \\
& \le (\theta + \|\ws\|_2 )\frac{2K}{\sqrt{L}}.
\end{align*}
The second line is derived by bounding the two factors separately. The first factor is bounded by the triangle inequality: $\norm{\hat{w}}_2 \le\norm{\ws}_2 +\norm{\hat{w}-\ws}_2 \le\norm{\ws}_2 +\theta$.
The second factor is a norm of an $L$-dimensional vector. The vector $\hat{y}_{t,\pi}$ has coordinates in $[0,K/L]$, whereas the coordinates
of $\bar{y}_t\one$, $y_\pi$, and $\bar{y}\vone$ are all in $[0,1]$, so the final vector has coordinates in $[-2,\,K/L+1]$,
and its Euclidean norm is thus at most $\sqrt{L}(2K/L)$ since $K \ge L$.

The variance can be bounded by the second moment, which is
\begin{align*}
\EE_{x,y,A}\Bracks{\left((\hat{y}_{\pi} - \bar{y}\mathbf{1})^T\hat{w}\right)^2}
&\le \|\hat{w}\|_2^2 \EE_{x,y,A}\Bracks{\sum_{\ell=1}^L (\hat{y}(\pi(x)_\ell) - \bar{y})^2}\\
&\le \|\hat{w}\|_2^2 \EE_{x,y,A}\Bracks{\sum_{\ell=1}^L \BigParens{\hat{y}(\pi(x)_\ell)^2 + \bar{y}^2}}\\
&\le \|\hat{w}\|_2^2 \sum_{\ell=1}^L\Parens{\frac{K}{L}\EE_{x,y,A}[\hat{y}(\pi(x)_\ell)]+1}\\
&=   \|\hat{w}\|_2^2 \sum_{\ell=1}^L\Parens{\frac{K}{L}\EE_{x,y}[y(\pi(x)_\ell)]+1}\\
&\leq  2(\theta + \|\ws\|_2 )^2 K,
\end{align*}
where the last inequality uses $K \geq L$. Bernstein's inequality
implies that with probability at least $1-\delta$, for all $\pi \in
\Pi$,
\[
\left|\frac{1}{n}\sum_{t=1}^n \BigParens{(\hat{y}_{t,\pi} - \bar{y}_t\mathbf{1})^T\hat{w}- \EE_{x,y}[y_\pi - \bar{y}\mathbf{1}]^T\hat{w}}\right|
\le (\theta+\|\ws\|_2 )\Bracks{\sqrt{\frac{4K\ln(2N/\delta)}{n}} + \frac{4K\ln(2N/\delta)}{3n\sqrt{L}}}.
\]

\textbf{Term 2:} For the second term, we use the Cauchy-Schwarz inequality,
\begin{align*}
\EE_{x,y}[y_{\pi} - \bar{y}\mathbf{1}]^T(\hat{w}-\ws) \le
\bigNorm{\EE_{x,y} [y_{\pi} - \bar{y}\mathbf{1}]}_2  \|\hat{w} - \ws\|_2
\end{align*}
The difference in the weight vectors will be controlled by our analysis of the least squares problem.
We need to bound the other quantity here and we will use two different bounds.
First,
\begin{align*}
\bigNorm{\EE_{x,y}[y_{\pi}- \bar{y}\mathbf{1}]}_2  \le \EE_{x,y}\|y_\pi - \bar{y}\mathbf{1}\|_2  \le \EE_{x,y}\|y_{\pi}\|_2  + \bar{y}\|\mathbf{1}\|_2  \le 2\sqrt{L}.
\end{align*}
Second,
\begin{align*}
\|\EE_{x,y}[y_{\pi}- \bar{y}\mathbf{1}]\|_2  & = \sqrt{ \EE_{x,y} \sum_{\ell=1}^L\bigParens{y(\pi(x)_\ell) - \bar{y}}^2}
\le  \sqrt{ \EE_{x,y} \sum_{a \in \Acal}(y(a) - \bar{y})^2}\\
& = \sqrt{K} \sqrt{\EE_{x,y} \frac{1}{K}\sum_{a \in \Acal}(y(a) - \bar{y})^2}
= \sqrt{KV}.
\end{align*}

\textbf{Combining everything:} Putting everything together, we obtain the bound
\begin{align*}
(\theta + \|\ws\|_2)\left[\sqrt{\frac{4K\ln(2N/\delta)}{n}} + \frac{4K\ln(2N/\delta)}{3n\sqrt{L}}\right] + \theta \min\{\sqrt{KV}, 2\sqrt{L}\}.
\end{align*}
Collecting terms together proves the main result.
\end{proof}

Assume that we explore for $\hat{t}$ rounds and then call \amo with weight vector $\hat{w}$ and importance-weighted rewards $\hat{y}_1, \ldots \hat{y}_{\hat{t}}$ to produce a policy $\hat{\pi}$ that maximizes $\eta_{\hat{t}}(\pi,\hat{w})$. In the remaining exploitation rounds we act according to $\hat{\pi}$.
With an application of Lemma~\ref{lem:eels_reward_deviation}, we can then bound the regret in the exploitation phase. Note that the algorithm ensures that $\hat{t}$ is at least equal to the
deterministic quantity $\ns$, so we can remove the dependence on the random variable $\hat{t}$:
\begin{lemma}[Exploitation Regret Lemma]
\label{lem:eels_exploitation_regret}
Assume that we explore for $\hat{t}$ rounds, where $\hat{t}\ge\ns$, and we find $\hat{w}$ satisfying $\|\hat{w} - \ws\|_2 \le \theta$. Then for any $\delta\in(0,1)$, with probability at least $1-2\delta$,
the exploitation regret is at most
\begin{align*}
\sum_{t=\hat{t}+1}^T\BigBracks{r_t(\pi^\star(x_t))-r_t(\hat{\pi}(x_t))}
&\le 4T(\theta + \|\ws\|_2)\sqrt{K}\left(\sqrt{\frac{\ln(2N/\delta)}{\ns}} + \sqrt{\frac{K}{L}}\frac{\ln(2N/\delta)}{\ns}\right)\\
& + 2T\theta\min\{\sqrt{KV},2\sqrt{L}\} + \|\ws\|_2\sqrt{2LT\ln(1/\delta)}.
\end{align*}
\end{lemma}
\begin{proof}
Using Lemma~\ref{lem:eels_reward_deviation} and the optimality of $\hat{w}$ for the importance-weighted rewards,
with probability at least $1-\delta$, the expected per-round regret of $\hat{\pi}$ is at most
\begin{align*}
&\eta(\pi^\star, \ws) - \eta(\hat{\pi}, \ws)
\\&\qquad{}
= \bigBracks{\eta(\pi^\star, \ws) - \psi_{\hat{t}}(\pi^\star, \hat{w})}
  + \bigBracks{\psi_{\hat{t}}(\hat{\pi}, \hat{w}) - \eta(\hat{\pi}, \ws)}
  + \bigBracks{\eta_{\hat{t}}(\pi^\star, \hat{w}) - \eta_{\hat{t}}(\hat{\pi}, \hat{w})}
\\&\qquad{}
\le 4(\theta + \|\ws\|_2)\sqrt{K}\left(\sqrt{\frac{\ln(2N/\delta)}{\hat{t}}} + \sqrt{\frac{K}{L}}\frac{\ln(2N/\delta)}{\hat{t}}\right) + 2\theta\min\{\sqrt{KV}, 2\sqrt{L}\}.
\end{align*}
To bound the actual exploitation regret, we use Hoeffding's inequality together with the fact that the absolute value of the per-round regret is at most $\|\ws\|_2\sqrt{L}$,
and finally apply bounds $1/\sqrt{\hat{t}}\le 1/\sqrt{\ns}$ and $1/\hat{t}\le 1/\ns$ to prove the lemma.
\end{proof}

\subsection{Proving the Final Bound}

The final bound will follow from regret bounds of
Lemmas~\ref{lem:eels_exploration_regret} and~\ref{lem:eels_exploitation_regret}. These bounds depend on parameters
$\ts$, $\ns$ and $\theta$. The parameter $\ns$ is specified directly by the algorithm and is assured to be a lower bound on the
stopping time. The parameter $\ts$ needs to be selected to upper-bound the stopping time
$\hat{t}$, and $\theta$ to upper-bound $\norm{\hat{w}-\ws}_2$.

\textbf{The stopping time bound $\ts$ and error bound $\theta$:}
Our algorithm uses the constants
\begin{align*}
&\ls \coloneqq \max\left\{6L^2\ln(4LT/\delta),\, (T\tilde{V}/B)^{2/3}\left(L\ln(2/\delta)\right)^{1/3}\right\}, \\ 
&\ns \coloneqq T^{2/3}(K\ln(N/\delta))^{1/3} \max\{L^{-1/3},\,(BL)^{-2/3}\}, \\ 
\intertext{and we will show we can set}
&\ts \coloneqq \max\left\{6\ls/V,\,\ns\right\}, \qquad 
\theta \coloneqq \sqrt{cL\ln(2/\delta)/\ls},
\end{align*}
where $c$ is the constant from \Lem{eels_weight_estimation}.

Recall that we assume $T \ge
(K\ln(N/\delta)/L)\max\{1,(B\sqrt{L})^{-2}\}$, which ensures that $T \ge
\ns$, and that the algorithm stops exploration with the first round
$\hat{t}$ such that $\hat{t}\ge\ns$ and
$\lambda_{\min}(\Sigma_{\hat{t}})>\ls$.  Thus, by
\Lem{eels_weight_estimation}, $\theta$ is indeed an upper bound on
$\norm{\hat{w}-\ws}_2$. Furthermore, since $\ts\ge\ns$, it suffices to
argue that $\Sigma_{\ts}\succeq\ls I_L$ with probability at least
$1-\delta$. We will show this through
\Lem{eels_covariance_conditioning}.

Specifically, \Lem{eels_covariance_conditioning} ensures that after $\ts$ rounds the 2nd-moment matrix satisfies, with probability at least $1-\delta$,
\begin{align*}
\Sigma_{\ts} \succeq \left(\ts V - 4L\sqrt{\ts V\ln(4LT/\delta)} - 4L\ln(4LT/\delta)\right)I_L.
\end{align*}
It suffices to verify that the expression in the parentheses is greater than $\ls$:
\begin{align*}
&\left(\ts V - 4L\sqrt{\ts V\ln(4LT/\delta)} - 4L\ln(4LT/\delta)\right) \ge \ls\\
&{\Leftarrow}\quad \left(\sqrt{\ts V} - 2L\sqrt{\ln(4LT/\delta)}\right)^2 - 4L^2\ln(4LT/\delta) - 4L\ln(4LT/\delta) \ge \ls\\
&{\Leftarrow}\quad \left(\sqrt{\ts V} - 2L\sqrt{\ln(4LT/\delta)}\right)^2 \ge \ls + 8L^2\ln(4LT/\delta)\\
&{\Leftarrow}\quad \sqrt{\ts V} - 2L\sqrt{\ln(4LT/\delta)} \ge \sqrt{\ls + 8L^2\ln(4LT/\delta)}\\
&{\Leftarrow}\quad \ts \ge \frac{1}{V}\left(\sqrt{\ls + 8L^2\ln(4LT/\delta)} + 2L\sqrt{\ln(4LT/\delta)}\right)^2
\end{align*}
Our setting is an upper bound on this quantity, using the inequality $(a+b)^2 \le 2a^2 + 2b^2$ and the fact that $\ls \ge 6L^2\ln(4LT/\delta)$.


\textbf{Regret decomposition:}
We next use Lemmas~\ref{lem:eels_exploration_regret} and~\ref{lem:eels_exploitation_regret} with the specific values of $\ts$, $\ns$ and $\theta$.
The leading term in our final regret bound will be on the order $T^{2/3}$. In the smaller-order terms, we ignore polynomial dependence on parameters other than $T$ (such as $K$ and $L$), which we make explicit via $\order_T$ notation, e.g., $O(\sqrt{LT})=\order_T(\sqrt{T})$.

The exploration regret is bounded by Lemma~\ref{lem:eels_exploration_regret}, using the bound $\ts\le 6\ls/V + \ns$, and the fact that the exploration vacuously stops at round $T$,
so $\ts$ can be replaced by $\min\set{\ts,T}$:
\begin{align*}
\textrm{Exploration Regret} & \le \min\{\ts, T\}\|\ws\|_2\min\{\sqrt{KV}, \sqrt{L}\} + \|\ws\|_2\sqrt{2LT\ln(1/\delta)}
\\[3pt]
& \le
  \underbrace{\min\Set{\frac{6\ls}{V}, T}B
  \min\{\sqrt{KV}, \sqrt{L}\}}_\text{Term 1}
+ \underbrace{\vphantom{\Set{\frac{6\ls}{V}}}
  \ns B\sqrt{L}}_\text{Term 2}
+ \order_T(\sqrt{T}).
\\[3pt]
\intertext{%
Meanwhile, for the exploitation regret, using the fact that $\ns=\Omega(T^{2/3})$, \Lem{eels_exploitation_regret} yields}
\text{Exploitation Regret}
& \le 4T(\theta + \|\ws\|_2)\sqrt{K}\left(\sqrt{\frac{\ln(2N/\delta)}{\ns}} + \sqrt{\frac{K}{L}}\frac{\ln(2N/\delta)}{\ns}\right)
\\[3pt]
&\qquad{}
  + 2T\theta\min\{\sqrt{KV},2\sqrt{L}\} + \|\ws\|_2\sqrt{2LT\ln(1/\delta)}
\\[3pt]
& = 4T(\theta + \|\ws\|_2)\sqrt{\frac{K\ln(2N/\delta)}{\ns}}
  + 2T\theta\min\{\sqrt{KV},2\sqrt{L}\} + \order_T(\sqrt{T})
\\[3pt]
& \le
  \underbrace{
     4T\Parens{\sqrt{\frac{cL\ln(2/\delta)}{\ls}} + B}
     \sqrt{\frac{K\ln(2N/\delta)}{\ns}}}_\text{Term 3}
\\[3pt]
&\qquad{}
  +
  \underbrace{
     2T\sqrt{\frac{cL\ln(2/\delta)}{\ls}}
     \min\{\sqrt{KV},2\sqrt{L}\}}_\text{Term 4}
  + \order_T(\sqrt{T}).
\end{align*}
We now use our settings of $\ns$ and $\ls$ to bound all the terms.
Working with $\ls$ is a bit delicate, because it relies on the estimate $\tilde{V}$ rather than $V$. However, by Lemma~\ref{lem:v_estimation} and \Eq{vtilde_bound}, we know that
\[
  V \le \tilde{V} \le 4V + \tau,
\]
where $\tau\coloneqq5\ln(2/\delta)/(2\ns)$.

\textbf{Term 1:}
We proceed by case analysis.
First assume that $V \le \tau$. Then
\begin{align*}
\textrm{Term 1} & \le T B \sqrt{KV} \le T B\sqrt{\frac{5K\ln(2/\delta)}{2\ns}} \le \textrm{Term 3},
\end{align*}
so we can use the bound on Term 3 to control this case.

Next assume
that $V\ge\tau$, which implies $\tilde{V}\le 5V$,
and distinguish two sub-cases. First, assume that $\ls$ is the second term in its definition,
i.e., $\ls=(T\tilde{V}/B)^{2/3}\left(L\ln(2/\delta)\right)^{1/3}$.
Then:
\begin{align*}
\textrm{Term 1}
&\le \frac{6\ls}{V}B\min\left\{\sqrt{K V}, \sqrt{L}\right\}\\
& \le \frac{6B (T\tilde{V}/B)^{2/3}(L\ln(2/\delta))^{1/3}\min\{\sqrt{KV}, \sqrt{L}\}}{V}\\
& \le 18 B^{1/3} T^{2/3}V^{-1/3} (L\ln(2/\delta))^{1/3} \min\{\sqrt{KV}, \sqrt{L}\},
\end{align*}
where the last step uses $\tilde{V}^{2/3}\le (5V)^{2/3}\le 3 V^{2/3}$.
We now show that
the term involving $V$ and the $\min\set{\cdot}$ is always bounded as follows:
\begin{claim}
\label{claim:blah}
$V^{-1/3}\min\{\sqrt{KV}, \sqrt{L}\} \le K^{1/3}L^{1/6}$.
\end{claim}
\begin{proof}
If $KV \le L$, then $V \le L/K$, and the expression equals $V^{-1/3}\sqrt{KV}=V^{1/6}\sqrt{K}\le L^{1/6}K^{1/3}$.
On the other hand, if $L \le KV$, then $V \ge L/K$, and the expression is equal to $V^{-1/3}\sqrt{L}\le K^{1/3}L^{1/6}$.
\end{proof}
Thus, in this case, Term 1 is $\order\left(T^{2/3} L^{1/2} (BK\log(2/\delta))^{1/3}\right)$.

Finally, assume that $\ls$ is the first term in its definition,
i.e.,
\[
  \ls=6L^2\ln(4LT/\delta)\ge(T\tilde{V}/B)^{2/3}\left(L\ln(2/\delta)\right)^{1/3},
\]
which implies
\begin{equation}
\label{eqn:v_ls_small}
 V \le \tilde{V} \le 6^{3/2}L^{5/2}\ln(4LT/\delta)^{3/2} (B/T) (\ln(2/\delta))^{-1/2}.
\end{equation}
Thus, we have the bound
\begin{align*}
\textrm{Term 1}
&
\le TB\sqrt{KV}
= \order\left(T^{1/2}B^{3/2}L^{5/4}K^{1/2}\log(LT/\delta)\right)
= \order_T(\sqrt{T}).
\end{align*}
In summary, we have the bound,
\begin{align}
\label{eq:t1_bd} \textrm{Term 1} \le \textrm{Term 3} +
\order\Parens{T^{2/3}L^{1/2}\bigParens{BK\log(2/\delta)}^{1/3}}+\order_T(\sqrt{T}).
\end{align}

\textbf{Term 2:}
Plugging in the definition of $\ns$ yields
\begin{align}
\label{eq:t2_bd}
\text{Term 2}
=
\ns B\sqrt{L} \le T^{2/3}(K\ln(N/\delta))^{1/3}\max\{BL^{1/6}, B^{1/3}L^{-1/6}\}.
\end{align}

\textbf{Term 3:}
Note that
\begin{gather*}
1/\sqrt{\ns}=T^{-1/3}(K\ln(N/\delta))^{-1/6}\min\{L^{1/6},(BL)^{1/3}\},
\\
\ls\ge 6L^2\ln(4LT/\delta)\ge L^2\ln(2/\delta),
\end{gather*}
so
\begin{align*}
\text{Term 3}
&=
4T\Parens{\sqrt{\frac{cL\ln(2/\delta)}{\ls}} + B}
     \sqrt{\frac{K\ln(2N/\delta)}{\ns}}\\
&\le
\order\Parens{
 T^{2/3}\Parens{\frac{1}{\sqrt{L}} + B}\bigParens{K\ln(N/\delta)}^{1/3}
 \min\set{L^{1/6},(BL)^{1/3}}
}.
\end{align*}
Now if $B \ge \frac{1}{\sqrt{L}}$, then the min above is achieved by the $L^{1/6}$ term, so the bound is
\begin{align*}
\order\left(T^{2/3}BL^{1/6}(K\log(N/\delta))^{1/3}\right).
\end{align*}
If $B \le \frac{1}{\sqrt{L}}$, then the min is achieved by the $(BL)^{1/3}$ term, so the bound is
\begin{align*}
\order\left(T^{2/3}B^{1/3}L^{-1/6}(K\log(N/\delta))^{1/3}\right).
\end{align*}
Thus,
\begin{align}
\text{Term 3}
=
\order\left(T^{2/3}(K\log(N/\delta))^{1/3}\left(BL^{1/6} + B^{1/3}L^{-1/6}\right)\right). \label{eq:t4_bd}
\end{align}

\textbf{Term 4:}
We distinguish two cases. If $\ls = 6L^2\ln(4LT/\delta)$ then Eq.~\eqref{eqn:v_ls_small} holds and thus
\[
  V=\order\Parens{ L^{5/2}\ln(4LT/\delta)^{3/2} (B/T) }.
\]
We then have
\begin{align*}
\textrm{Term 4}
&= 2T\sqrt{\frac{cL\ln(2/\delta)}{\ls}}\min\{\sqrt{KV}, 2\sqrt{L}\}
\\
&\le
  \order\Parens{\frac{T}{\sqrt{L}}\sqrt{KV}}
= \order\left(\sqrt{BTK}L^{3/4}\ln(LT/\delta)\right)
= \order_T(\sqrt{T}).
\end{align*}
Otherwise,
$\ls
=   (T\tilde{V}/B)^{2/3}\left(L\ln(2/\delta)\right)^{1/3}$, and since
$\tilde{V}\ge V$, we obtain
\begin{align}
\textrm{Term 4} &= 2T\sqrt{\frac{cL\ln(2/\delta)}{\ls}}\min\{\sqrt{KV}, 2\sqrt{L}\} \notag \\
& \le
\order\left(T\sqrt{\frac{L\log(2/\delta)}{(TV/B)^{2/3}(L\log(2/\delta))^{1/3}}}
\min\{\sqrt{KV}, \sqrt{L}\}\right) \notag\\
&= \order\left(T^{2/3}B^{1/3}L^{1/3}V^{-1/3}(\log(2/\delta))^{1/3}\min\{\sqrt{KV},\sqrt{L}\}\right) \notag\\
& = \order\left(T^{2/3}(BK)^{1/3}L^{1/2}(\log(2/\delta))^{1/3}\right),\label{eq:t6_bd}
\end{align}
where the last step is Claim~\ref{claim:blah}. This is the
leading-order term since the other cases are $\order_T(\sqrt{T})$.

\textbf{Putting everything together:}
Combining Eqs.~\eqref{eq:t1_bd},~\eqref{eq:t2_bd},~\eqref{eq:t4_bd}, and~\eqref{eq:t6_bd}, we
obtain the bound on the sum of the exploration and exploitation regret:
\begin{align*}
\textrm{Regret} =
\order\left(T^{2/3}(K\log(N/\delta))^{1/3}\max\{B^{1/3}L^{1/2}, BL^{1/6}\}\right)
+ \order_T(\sqrt{T}).
\end{align*}

\section{Deviation Bounds}
\label{sec:deviations}

Here, we collect several deviation bounds that we use in our proofs.
All of these results are well known and we point to references rather than provide proofs.
The first inequality, which is a Bernstein-type deviation bound for martingales, is Freedman's inequality, taken from Beygelzimer et. al~\citeapp{beygelzimer2011contextual}
\begin{lemma}[Freedman's Inequality]
\label{lem:freedman}
Let $X_1, X_2, \ldots, X_T$ be a sequence of real-valued random variables.
Assume for all $t \in \{1, 2, \ldots, T\}$ that $X_t \le R$ and $\EE[X_t | X_1, \ldots, X_{t-1}] = 0$.
Define $S = \sum_{t=1}^T X_t$ and $V = \sum_{t=1}^T\EE[X_t^2 | X_1, \ldots, X_{t-1}]$.
For any $\delta \in (0,1)$ and $\lambda \in [0, 1/R]$, with probability at least $1-\delta$:
\begin{align*}
S \le (e-2)\lambda V + \frac{\ln(1/\delta)}{\lambda}
\end{align*}
\end{lemma}

We also use Azuma's inequality, a Hoeffding-type inequality for martingales.
\begin{lemma}[Azuma's Inequality]
\label{lem:azuma}
Let $X_1, X_2, \ldots, X_T$ be a sequence of real-valued random variables.
Assume for all $t \in \{1, 2, \ldots, T\}$ that $X_t \le R$ and $\EE[X_t | X_1, \ldots, X_{t-1}] = 0$.
Define $S = \sum_{t=1}^T X_t$.
For any $\delta \in (0,1)$, with probability at least $1-\delta$:
\begin{align*}
S \le R\sqrt{2 T \ln(1/\delta)}
\end{align*}
\end{lemma}

We also make use of a vector-valued version of Hoeffding's inequality, known as the Hanson-Wright inequality, due to Rudelson and Vershynin~\citeapp{rudelson2013hanson}.
\begin{lemma}[Hanson-Wright Inequality~\citeapp{rudelson2013hanson}]
\label{lem:vector_concentration}
Let $X = (X_1, \ldots, X_n)$ be a random vector with independent components satisfying $\EE X_i = 0$ and $|X_i| \le \kappa$ almost surely.
There exists a universal constant $c_0>0$ such that, for any $A \in \RR^{n \times n}$ and any $\delta \in (0,1)$, with probability at least $1-\delta$,
\begin{align*}
|X^TAX - \EE X^TAX| \le \kappa^2\sqrt{c_0\|A\|_F^2\log(2/\delta)}  + c_0\kappa^2\|A\|\log(2/\delta),
\end{align*}
where $\|\cdot\|_F$ is the Frobenius norm and $\norm{\cdot}$ is the spectral norm.
\end{lemma}

Finally, we use a well known matrix-valued version of Bernstein's inequality, taken from Tropp~\citeapp{tropp2011user}.
\begin{lemma}[Matrix Bernstein]
\label{lem:matrix_bernstein}
Consider a finite sequence $\{X_k\}$ of independent, random, self-adjoint matrices with dimension $d$.
Assume that for each random matrix we have $\EE X_k = 0$ and $\lambda_{\max}(X_k) \le R$ almost surely.
Then for any $\delta \in (0,1)$, with probability at least $1-\delta$:
\begin{align*}
\lambda_{\max}\BigParens{\sum_k X_k} \le \sqrt{2\sigma^2 \ln(d/\delta)} + \frac{2}{3}R\log(d/\delta) \qquad \textrm{with} \qquad \sigma^2 = \BigNorm{\sum_k \EE(X_k^2)},
\end{align*}
where $\norm{\cdot}$ is the spectral norm.
\end{lemma}

\bibliographystyle{plainnat}
\bibliography{bibliography}

\end{document}